\documentclass{article}

    \PassOptionsToPackage{numbers, compress}{natbib}

\usepackage{amsmath,amsfonts,bm}

\def\eqref#1{equation~\ref{#1}}

\def\1{\bm{1}}

\DeclareMathAlphabet{\mathsfit}{\encodingdefault}{\sfdefault}{m}{sl}
\SetMathAlphabet{\mathsfit}{bold}{\encodingdefault}{\sfdefault}{bx}{n}

\newcommand{\R}{\mathbb{R}}

\usepackage[preprint]{neurips_2024}

\usepackage[utf8]{inputenc} 
\usepackage[T1]{fontenc}    
\usepackage{hyperref}       
\usepackage{url}            
\usepackage{booktabs}       
\usepackage{amsfonts}       
\usepackage{nicefrac}       
\usepackage{microtype}      
\usepackage{xcolor}         
\usepackage{amsmath}
\usepackage{amssymb}
\usepackage{mathtools}
\usepackage{amsthm}
\usepackage{thm-restate}
\usepackage{multirow}
\usepackage{cleveref}

\theoremstyle{plain}
\newtheorem{theorem}{Theorem}[section]
\newtheorem{proposition}[theorem]{Proposition}

\theoremstyle{definition}
\newtheorem{definition}[theorem]{Definition}

\theoremstyle{remark}
\newtheorem{remark}[theorem]{Remark}

\title{Almost Equivariance via Lie Algebra Convolutions}

\author{%
  Daniel McNeela\\
  Department of Computer Sciences\\
  University of Wisconsin, Madison\\
  Madison, WI 53703\\
  \texttt{mcneela@wisc.edu} \\
}

\begin{document}

\maketitle

\begin{abstract}
  Recently, {\it equivariant neural networks\/} have become an important topic of research in machine learning. 
  However, imbuing an architecture with a specific group equivariance imposes a strong prior on the types of data 
  transformations that the model expects to see. While strictly-equivariant models enforce symmetries, 
  real-world data does not always conform to such strict equivariances.
  In such cases, the prior of strict equivariance can actually prove too strong and cause models 
  to underfit on real-world data. Therefore, in this work we study a closely related topic, 
  that of {\it almost equivariance}. We provide a definition of almost equivariance that 
  differs from those extant in the current literature
  and give a practical method for encoding almost equivariance in models by appealing to the Lie algebra of a Lie group.
  Specifically, we define {\it Lie algebra convolutions} and demonstrate that they offer several benefits over Lie group convolutions, 
  including being well-defined for non-compact Lie groups having non-surjective exponential map.
  From there, we pivot to the realm of theory and demonstrate parallel connections between the notions of equivariance and isometry and 
  those of almost equivariance and almost isometry. Finally, we demonstrate the validity of our approach by 
  benchmarking against datasets in fully equivariant and almost equivariant settings.
\end{abstract}

\section{Introduction}
The past few years have shown a surge in interest in {\it equivariant\/} model architectures,
those that explicitly impose symmetry with respect to a particular group acting on the 
model inputs. While data augmentation strategies have been proposed to make generic models exhibit greater symmetry
without the need for equivariant model architectures, much work has demonstrated that this is an inefficient
approach at best \citep{gerken22a,lafarge,wang2022data}. As such, developing methods for building neural network layers 
that are equivariant to general group actions is of great importance. 

More recently, {\it almost equivariance}, also referred to variously  as {\it approximate, soft,} or 
{\it partial equivariance}, has become a rich topic of study. The idea is that the symmetry 
constraints imposed by full equivariance are not always completely conformed to in real-world systems. 
For example, the introduction of external forces and certain boundary conditions 
into models of turbulence and fluid flow break many theoretical symmetry constraints. Accurately
modeling real-world physical systems therefore requires building model architectures 
that have a built-in notion of symmetry but that are not so constrained by it as to be incapable
of fully modeling the underlying system dynamics.

\section{Related Work}
\subsection{Strict Equivariance}
Much of the work in developing strictly-equivariant model architectures began with the 
seminal paper of \citet{cohenc16}, which introduced the group equivariant convolutional 
neural network layer. \citet{kondor18a} generalized this notion of equivariance and convolution
to the action of an arbitrary compact group. Further generalizations followed,
with the creation of convolutions \citep{finzi2020generalizing} and efficient MLP layers \citep{finzi21a} 
equivariant to arbitrary Lie groups. Other neural network types have also been studied 
through the lens of equivariance, for example, graph neural networks \citep{satorras21a}, \citep{nequip},
transformers \citep{2021lietransformer}, and graph transformers \citep{liao2023equiformer}. 
\citet{cohen2020general} consolidated much of this
work into a general framework via which equivariant layers can be understood as maps between spaces
of sections of vector bundles. Similar to our work, \citet{nimaalgebra} devised a convolutional layer on the 
Lie algebra designed to approximate group convolutional layers. 
However, their objective was to make the layer as close to equivariant as possible whereas our layer is designed to
be flexible so as to be capable of modelling almost equivariances.
Finally, rather than devising a new equivariant 
layer type, \citet{gruver2023the} developed a method based on the Lie derivative which
can be used to detect the degree of equivariance learned by an arbitrary model architecture.

\subsection{Almost Equivariance}
One of the first works on almost equivariance was \citet{finzi2021residual}, which introduced 
the {\it Residual Pathway Prior\/}. Their idea is to construct a neural network 
layer, $f$, that is the sum of two components, $A$ and $B$, where $A$ is a strictly equivariant 
layer and $B$ is a more flexible, non-equivariant layer
\[
  f(x) = A(x) + B(x)
\]
Furthermore, they place priors on the sizes of $A$ and $B$ such that a model trained using 
maximum a posteriori estimation is incentivized to favor the strict equivariance of $A$
while relying on $B$ only to explain the difference between $f$ and the fully symmetric 
prediction function determined by $A$. The priors on $A$ and $B$ can be defined 
so as to weight the layer towards favoring the use of $A$.

The approach taken in \citet{wang2022approximately} is somewhat different. They give an 
explicit definition of approximate equivariance, then model it via a
{\it relaxed group convolutional layer\/} wherein the single kernel, $\Psi$, of a strictly 
equivariant group convolutional layer is replaced with a set of kernels, ${\{\Psi_l\}}_{l=1}^L$.
This introduces a specific, symmetry-breaking dependence on a pair of group elements, $(g, h)$,
i.e.
\[
  \Psi(g, h) = \sum_{l=1}^L w_l(h) \Psi_l (g^{-1}h)
\]
Their full relaxed group convolution operation is then defined as follows
\begin{align*}
  [f_{\star_G}\Psi](g) &= \sum_{h \in G} f(h)\Psi(g, h) \\
  &= \sum_{h \in G} \sum_{l=1}^L f(h) w_l(h) \Psi_l (g^{-1}h)
\end{align*}

\citet{romero2022learning} take an altogether different approach.
They introduce a model, which they call the {\it Partial G-CNN},
and show how to train it to learn layer-wise levels of equivariance from data.
A key differentiator in their method is the learning of a probability distribution
over group elements at each group convolutional layer, allowing them to sample group 
elements during group convolutions.

More specifically, they define a {\it $G$-partially equivariant map}, $\phi$, as one 
that satisfies
\[
  \phi(g \cdot x) = g \cdot \phi(x) \quad \forall x \in X, g \in S
\]
where $S$ is a subset, but not necessarily a subgroup, of $G$. They then define 
a partial group convolution from $f: S^1 \to \R$ to $h: S^2 \to \R$ as 
\begin{align*}
  h(u) = (\psi \star f)(u) = &\int_{S^1} p(u)\psi(v^{-1}u)f(v)\ d\mu_G(v); 
\end{align*}
for $u \in S^2, v \in S^1$, where $p(u)$ is a probability distribution on $G$ and $\mu_G$ is the Haar measure.

In order to learn the convolution for one-dimensional, continuous groups, they parameterize 
$p(u)$ by applying a reparameterization trick to the Lie algebra of $G$. This allows them to 
define a distribution which is uniform over a connected set of group elements, 
$[u^{-1}, \ldots, e, \ldots, u]$, but zero otherwise. Thus they define a uniform 
distribution, $\mathcal{U}(\mathfrak{u} \cdot [-1, 1])$, with learnable $\mathfrak{u} \subset \mathfrak{g}$
and map it to the group via the exponential map, $\exp : \mathfrak{g} \to G$.

\citet{ouderaa2022relaxing} relax equivariance constraints by defining a non-stationary group convolution
\[
  h(u) = (k_{\theta} \star f)(u) = \int_G k_{\theta}(v^{-1}(u), v)f(v)\ d\mu(v)
\]
They parameterize the kernel by choosing a basis for the Lie algebra, $\mathfrak{g}$, of $G$ and defining elements, 
$g \in G$, as exponential maps of Lie algebra elements, i.e.
\[
  g = \text{exp}(a) = \text{exp}\left( \sum_{i=1}^n \alpha_i A_i \right)
\]
where $a \in \mathfrak{g}$ and $\{A_i\}$ is a basis for $\mathfrak{g}$.
In particular, they achieve fine-grained control over the kernel representation by choosing a basis of 
{\it Random Fourier Features (RFF)\/} for $\mathfrak{g}$.

Finally, \citet{petrache2023approximationgeneralization} provide a take on approximate equivariance rooted in 
statistical learning theory and provide generalization and error bounds on approximately equivariant architectures.

\subsection{Notions of Approximate Equivariance}
It's worthwhile to note that there are multiple notions of approximate, partial, and soft equivariance, only some of which we explicitly
address in this work.

The first type occurs when we only have partially observed data, for example, a single pose of a 3D object captured in a 2D image or an 
object occlusion in computer vision. \citet{wang2023a} refer to this as \textit{extrinsic equivariance} in that applying a group transformation
to an in-distribution data point transforms it to an out-of-distribution data point. This type of partial equivariance is often addressed via data augmentation.
We do not explicitly test our approach in this setting.

The second type occurs when we have noise in data that breaks equivariance. This is one setting we explicitly address.

The third type occurs when we have data that naturally exhibits almost equivariance. For example, data sampled from vector fields 
and PDEs governing natural physical processes often exhibit this quality. This is another setting we explicitly address.

Finally, there is what \citet{wang2023a} call \textit{incorrect equivariance}. This occurs when applying a group transformation to 
a data point qualitatively and quantitatively changes its label. For example, rotating the digit 6 by 180 degrees turns it into 
the digit 9 and vice versa. We do not explicitly address this in our method, but our model performs competitively on the Rot-MNIST 
classification task, indicating that it has the capability of accounting for incorrect equivariances in its modeling.

\section{Theory}
\subsection{Equivariance \& Almost Equivariance}
In this section, we seek to give a suitable definition of almost equivariance and establish the relationship between 
it and full equivariance. In defining almost equivariance of a model with respect to the action of some Lie group, $G$, 
we seek a definition that offers both theoretical insight as well as practical significance. We start by addressing 
the abstract case, in which we define almost equivariance for general functions on a Riemannian manifold. We then
drop to the level of practice and give a method for encoding almost equivariance into a machine learning model 
taking inputs on some data manifold. 

\begin{definition}
Let $G$ be a Lie group acting smoothly on smooth Riemannian manifolds $(M, g)$ and $(N, h)$ via the left actions $G \times M \to M$ and $G \times N \to N$
given by $(g, x) \mapsto g \cdot x$. Furthermore, let $f$ be a mapping of smooth manifolds, $f: M \to N$.
Then we say $f$ is \textit{equivariant} with respect to the action of $G$ if it commutes with the actions of $G$ on $M$ and $N$, i.e.
\[
\vspace{0.1cm}
  g \cdot f(x) = f(g \cdot x)
\vspace{0.1cm}
\]
\end{definition}
\begin{definition}
  Now, consider the same setup as in the previous definition. We say a function $f : M \to N$ is $\varepsilon$-almost equivariant if
  the following is satisfied
  \[
      d(f(g \cdot x), g \cdot f(x)) < \varepsilon
  \]
  for all $g \in G$ and $x \in M$, where $d$ is the distance metric on $N$. We can think of such a function as commuting with the actions of $G$ on $M$ and $N$ to within some $\varepsilon$.
\end{definition}

This definition is reminiscent of one given in a question posed by Stanislaw Ulam \citep{Ulam1960} concerning the stability of 
certain ``quasi'' group homomorphisms. In particular, given a group $\Gamma$, a group $G$ equipped with a distance $d$,
and a $\delta$-homomorphism, $\mu : \Gamma \to G$, satisfying 
\[
  d(\mu(xy), \mu(x)\mu(y)) < \delta
\]
for all $x, y \in \Gamma$, he asked whether there exists an actual group homorphism that is ``close'' to $\mu$ with respect to the distance, $d$. 
This question spurred research that showed the answer to be in the affirmative in a variety of cases, given certain restrictions on $\mu, \Gamma$, and $G$. 

In our case, we seek to address a similar question, that is, whether given an almost equivariant map as defined above, 
there exists a fully equivariant map that is ``close'' to it in the sense of being within some bounded distance, and vice versa.
If such maps do exist, we hope to determine under what conditions on $G$ and $M$ they can be found.

\subsection{Isometries, Isometry Groups, and Almost Isometries}
We begin our discussion of the theory underlying almost equivariance by studying the notions of {\it isometry\/} and {\it almost isometry}.
Because we often seek to impose in our models equivariance with respect to the action of the isometry group of a manifold from which data is sampled, we find it 
worthwhile to study isometries as a precursor to studying equivariance. 
An isometry is a mapping of metric spaces that preserves the distance metric. 
Some common types of metric spaces for which there exists a natural notion of isometry
are {\it normed spaces\/}, such as {\it Banach spaces\/}, and {\it Riemannian manifolds}.
In this work, we focus most of our analysis on Riemannian manifolds, as they are among 
the most general spaces upon which equivariant models operate \citep{bronstein2021geometric}.

\begin{definition}
  Let $(M, g)$ and $(\tilde{M}, \tilde{g})$ be Riemannian manifolds adding
  $\varphi : M \to \tilde{M}$ a diffeomorphism. Then we say $\varphi$ is an 
  {\it isometry of $M$\/} if $g = \varphi^* \tilde{g}$. In other words, the metric $\tilde{g}$
  can be pulled back by $\varphi$ to get the metric $g$.
\end{definition}

Next, we give a definition of an {\it$\varepsilon$-almost isometry\/}, which, in close analogy with almost equivariance,
is a mapping of manifolds that preserves the metric on a Riemannian manifold, $M$, to within some $\varepsilon$.

\begin{definition}
  Let $(M, g)$ and $(M, \tilde{g})$ be Riemannian manifolds, 
  $\varphi : M \to M$ a diffeomorphism, and $\varepsilon > 0$. 
  Then we say $\varphi$ is an {\it $\varepsilon$-almost isometry of $M$\/} if
  \[
      \left|(g - \varphi^* \tilde{g})_p(v, w)\right| < \varepsilon
  \]
for any $p \in M$ and any $v, w \in T_pM$ having unit norm. 
\end{definition}

In other words, $\varepsilon$-almost isometries are maps between the same Riemannian manifold equipped with two different metrics
for which the metric on pairs of vectors, $g_p(v, w)$, and the metric on their pushforward by $\varphi$,
$\tilde{g}_{\varphi(p)}(d\varphi_p(v), d\varphi_p(w))$, differ by at most $\varepsilon$.

Our definition of an $\varepsilon$-almost isometry is \textit{local} in the sense that it deals with the tangent spaces to 
points $p, \varphi(p)$ of a Riemannian manifold. However, we can naturally extend this definition to a \textit{global} version
that operates on vector fields. The local and global definitions are related by the following fact: 
if $(M, g)$ is locally $\varepsilon$-almost isometric to $(M, \tilde{g})$ via $\varphi$, 
then globally it is at most $(\varepsilon \cdot \text{Vol}_g(M))$-isometric.

\begin{definition}
  Given oriented and compact Riemannian manifolds $(M, g)$ and $(M, \tilde{g})$ and a local $\varepsilon$-almost isometry, $\varphi: M \to M$,
  we say that $\varphi$ is a \textit{global $E$-almost isometry} if there exists a continuous, compactly-supported scalar field, $E : M \to \mathbb{R}$, 
  such that for any normalized vector fields $X, Y \in \Gamma(TM)$, we have
  $$\left|(g - \varphi^*\tilde{g})(X, Y)\right| < E$$
  In particular, $\int_M E\ \omega_g \leq \varepsilon \cdot \text{Vol}_g(M)$,
  where $\omega_g$ is the canonical Riemannian volume form on $(M, g)$.
\end{definition}

It is known that equivariant model architectures are designed to preserve symmetries in data by imposing equivariance with respect to a Lie group action. 
Typical examples of Lie groups include the group of $n$-dimensional rotations, $SO(n)$, the group of area-preserving transformations, $SL(n)$, and the special unitary group,
$SU(n)$, which has applications to quantum computing and particle physics. 
Some of these Lie groups are, in fact, \textit{isometry groups} of the underlying manifolds from which data are sampled. 
\begin{definition}
  The {\it isometry group\/} of a Riemannian manifold, $M$, is the set of isometries 
  $\varphi: M \to M$ where the group operations of multiplication and inversion are given by function
  composition and function inversion, respectively. In particular, the composition of two isometries is 
  an isometry, and the inverse of an isometry is an isometry. We denote the isometry group of $M$ by $\text{Iso}(M)$ and the 
  set of $\varepsilon$-almost isometries of $M$ by $\text{Iso}_{\varepsilon}(M)$.
\end{definition}

To give some examples, $E(n) = \R^n \rtimes O(n)$ is the isometry group of $\R^n$, while the Poincaré group, $\R^{1,3} \rtimes O(1, 3)$,
is the isometry group of Minkowski space, which has important applications to special and general relativity. 
We often seek to impose equivariance in our models with respect to such isometry groups.
Isometry groups of Riemannian manifolds also satisfy the following deep theorem, due to Myers and Steenrod.

\begin{theorem}[\citet{myersSteenrod}]\label{theorem:myerssteenrod}
    The isometry group of a Riemannian manifold is a Lie group. 
\end{theorem}

Thus, we can apply all the standard theorems of Lie theory to the study of isometry groups. 
Using basic facts, we can deduce the following result about equivariance.

\begin{remark}\label{proposition:abelianequivariance}
  If $f: M \to M$ is an isometry of the Riemannian manifold $(M, g)$ and $\text{Iso}(M)$ be abelian,
  then $\text{Iso}(M)$ acts smoothly on $M$ and $f$ is an equivariant map with respect to this action 
  of $\text{Iso}(M)$ on $M$. To see why, note that since $\text{Iso}(M)$ is abelian, we have by definition that
  $g \cdot f = f \cdot g$ for all $g \in \text{Iso}(M)$, which shows that $f$ is equivariant with respect to the 
  action of $\text{Iso}(M)$ on $M$.
\end{remark}

However, we cannot, without some work, consider this 
theorem in the context of $\text{Iso}_{\varepsilon}(M)$ because the set of $\varepsilon$-almost isometries of a manifold 
does not form a group. To see why, note that composing two $\varepsilon$-almost isometries produces, in general, a $2\varepsilon$-almost isometry, 
thus the set of $\varepsilon$-almost isometries of a manifold is not closed under composition. 
Still, we can impose the abelian condition on group actions as a stepping stone 
towards studying more general group actions, almost isometries, and equivariant functions. Under the assumption of an abelian Lie group acting on a Riemannian manifold,
we prove the following theorem.

\begin{theorem}
  Let $(M, g)$ be a Riemannian manifold and suppose its group of isometries, $G = \text{Iso}(M)$, is an abelian Lie group.
  Let $f \in \text{Iso}(M)$, and suppose there exists a continuous $\varepsilon$-almost isometry, $f_{\varepsilon} \in \text{Iso}_{\varepsilon}(M)$,
  with $f \neq f_{\varepsilon}$, such that 
  \[
      \sup_{p \in M} d(f(p), f_{\varepsilon}(p)) < \varepsilon
  \]
  where we abbreviate the above as $d(f, f_{\varepsilon})$ on $C^{\infty}(M, M)$ 
  and interpret it as an analogue to the supremum norm on the space of real-valued functions on $M$, i.e. $C^{\infty}(M)$. 
  Then $f_{\varepsilon}$ is $2\varepsilon$-almost equivariant with respect
  to the action of $G$ on $M$. That is, it satisfies 
  \[
      d(g \cdot f_{\varepsilon}(x), f_{\varepsilon}(g \cdot x)) < 2\varepsilon
  \]
  for any $g \in \text{Iso}(M)$ and any $x \in M$.
\end{theorem}
\begin{proof}
  By Proposition~\ref{proposition:abelianequivariance}, since $\text{Iso}(M)$ is abelian, any $f \in \text{Iso}(M)$ is equivariant to actions of 
  $\text{Iso}(M)$, i.e. we have
  \[
      g \cdot f(x) = f(g \cdot x)
  \]
  for all $x \in M$. Equivalently, we have $d(g \cdot f(x), f(g \cdot x)) = 0$. Now, $d(f(x), f_{\varepsilon}(x)) < \varepsilon$ by definition 
  of the supremum norm. Then, we have $d(f(g \cdot x), f_{\varepsilon}(g \cdot x)) < \varepsilon$ simply by definition of a $\varepsilon$-almost isometry. Since $g$ is an isometry, it preserves distances, so we have $d(g \cdot f(x), g \cdot f_{\varepsilon}(x)) < \varepsilon$ because
  $d(f(x), f_{\varepsilon}(x)) < \varepsilon$.
   
  Using the fact that $f$ is equivariant to actions of $g \in \text{Iso}(M)$ and applying the inequalities just derived, along with repeated applications of the triangle inequality, we get
  \begin{align}
      d(g \cdot f(x), f_{\varepsilon}(g \cdot x)) &< d(g \cdot f(x), f(g \cdot x)) + \\
      d(f(g \cdot x), f_{\varepsilon}(g \cdot x)) \\
      &< 0 + \varepsilon = \varepsilon \\
      d(g \cdot f_{\varepsilon}(x), f_{\varepsilon}(g \cdot x)) &< d( g \cdot f_{\varepsilon}(x), g \cdot f(x)) + \\
      d(g \cdot f(x), f_{\varepsilon}(g \cdot x)) \\
      &< \varepsilon + \varepsilon = 2\varepsilon
  \end{align}
  We apply the triangle inequality to get (1) and (4), and we substitute the inequalities derived above to get (3) and (6).
  Thus, $d(g \cdot f_{\varepsilon}(x), f_{\varepsilon}(g \cdot x)) < 2\varepsilon$, which shows that $f_{\varepsilon}$ is 
  $2\varepsilon$-almost equivariant with respect to the action of $\text{Iso}(M)$. This completes the proof.
\end{proof}

Of course, this theorem is not particularly useful unless for every isometry, $f \in \text{Iso}(M)$, 
we have a way of obtaining an $\varepsilon$-almost isometry, $f_{\varepsilon} \in \text{Iso}_{\varepsilon}(M)$,
satisfying
\[
  d(f, f_{\varepsilon}) < \varepsilon
\]
The next theorem shows that such $f_{\varepsilon}$ are plentiful. In fact, there are infinitely many of them.
Furthermore, not only can we find $f_{\varepsilon} : M \to M$, but we can find an isometric embedding, $\varphi : M \to \R^n$,
of $f$ that is $G$-equivariant and then construct $f_{\varepsilon} : M \to \R^n$ as an $\varepsilon$-almost isometric 
embedding of $M$ into $\R^n$ such that
\[
  \| \varphi(f) - f_{\varepsilon} \|_{\infty} < \varepsilon
\] 
This is particularly useful in the context of machine learning, where we normally appeal to 
embedding abstract manifolds into some discretized subspace of $\R^n$ in order to actually perform computations on a finite-precision computer.
We then later give some conditions under which we can achieve the converse, that is, given an $\varepsilon$-almost isometry, 
$f_{\varepsilon}$, of a metric space, $X$, find an isometry, $f$ of $X$, such that 
\[
  \sup_{x \in X} d(f(x), f_{\varepsilon}(x)) < c \cdot \varepsilon
\] 
for some constant $c \in \R$.

\begin{restatable}{lemma}{existencealmostequiv}\label{proposition:existencealmostequiv}
  Let $(M, g)$ be a compact Riemannian manifold without boundary, $G$
  a compact Lie group acting on $M$ by isometries, $f: M \to M$ a $G$-equivariant function, 
  and $\varepsilon > 0$. Then there exists an orthogonal representation $\rho$ of $G$,
  i.e. a Lie group homomorphism from $G$ into the orthogonal group $O(N)$ which acts on $\R^n$ by rotations and reflections,
  an isometric embedding $\varphi: M \to \mathbb{R}^N$, and an $\varepsilon$-almost isometric embedding, $f_{\varepsilon}: M \to \mathbb{R}^N$, 
  such that $\varphi$ is equivariant with respect to $\rho$, i.e.
  \[
     \rho(g) \cdot \varphi(f(x)) = \varphi(g \cdot f(x)), \quad \text{for } g \in G
  \]
  and $f_{\varepsilon}$ is $\varepsilon$-almost isometric with respect to $\varphi(f)$, i.e. it satisfies
  \begin{equation*}
      \| \varphi(f) - f_{\varepsilon} \|_{\infty} < \varepsilon
  \end{equation*}
\end{restatable}
\begin{proof}
  Under the stated assumptions of $M$ a compact Riemannian manifold and $G$ a compact Lie group 
  acting on $M$ by isometries, we can get the existence of $\rho$ and $\varphi$ by invoking the 
  main theorem of \citet{equivariantnash}. From there, note that setting $f_{\varepsilon} = \varphi(f)$ 
  trivially satisfies (1) for any $\varepsilon$, although we seek a non-trivial solution.
  We can choose an arbitrary $x_0 \in M$, and define $f_{\varepsilon}(x) = \varphi(f(x))$
  for all $x \neq x_0 \in M$. Next, since $M$ is compact, $\varphi(f)$
  is bounded on $M$. We can then take a neighborhood $U$ of $x_0$ such that 
  $B_{\varepsilon}(\varphi(f(x_0))) \subseteq \varphi(U)$. We can then choose an arbitrary $y \in B_{\varepsilon}(\varphi(f(x_0)))$,
  while requiring $y \neq \varphi(f(x_0))$, and set $f_\varepsilon(x_0) = y$. Then $f_{\varepsilon}$ is 
  an $\varepsilon$-almost isometric embedding of $M$, but $f_{\varepsilon} \neq f$, as desired.
  Furthermore, given a suitable topology on $C^{\infty}(M)$ (such as the compact-open topology),
  $B_{\varepsilon}(\varphi(f(x_0)))$ is open so that there exist infinitely many such $f_{\varepsilon} \neq f$,
  and they can be taken to be continuous.
\end{proof}

We've now shown that, subject to restrictions on $G$,
given a $G$-equivariant isometry of $M$, $f$, we can find $\varepsilon$-isometries of $M$,
$f_{\varepsilon}$, within distance $\varepsilon$ to $f$ that are, in fact, $2\varepsilon$-almost equivariant with respect
to the $G$-action on $M$.
The next, more difficult question (Theorem~\ref{theorem:ulamstability}) concerns a partial converse.
That is, given an $\varepsilon$-almost isometry, $f_{\varepsilon}$, can we find an isometry, $f$, that
differs from $f_{\varepsilon}$ by no more than some constant multiple of $\varepsilon$, for all inputs $x$? The answer here 
is, {\bf yes}, but proving it takes some work. We address this question in the next section. 

\begin{figure}[t]
  \centering
  \includegraphics[width=.5\textwidth]{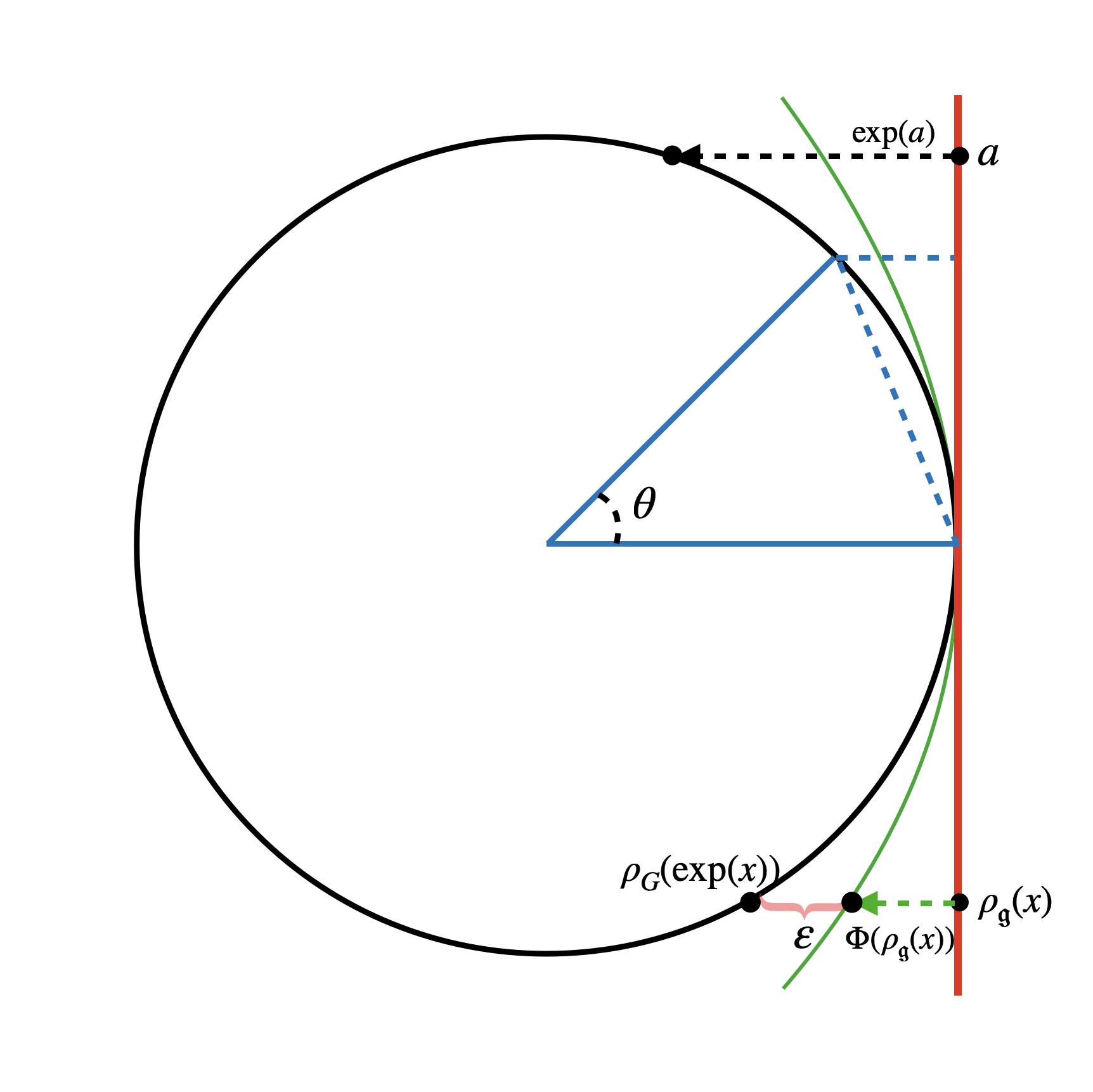}
  \caption{We provide a visualization of how actions of the Lie algebra can be used 
  to approximate actions of the corresponding Lie group. The Lie group, $SO(2)$, of 
  two-dimensional rotations, represented here as the circle, $S^1 \subset \mathbb{C}^2$, with
  its Lie algebra, $\mathfrak{so}(2)$, represented here as the tangent line at the 
  identity $x = 1 \in \mathbb{C}$, is the most easily visualized case. Here, 
  $\theta$ gives the angle of rotation, $\varepsilon$ gives the approximation
  error arising from working in the Lie algebra, and the top dashed arrow shows how points 
  can be mapped from the Lie algebra onto the Lie group via the exponential map.
  The function $\Phi: \mathfrak{g} \to  G$, indicated here by the green curve,
  is a learned mapping that can be trained to approximate the exponential map.}
  \label{fig:intuition}
\end{figure}

\subsection{Ulam Stability Theory}
There exist a number of results in the mathematics literature that confirm the 
existence of almost isometries that are ``close'' to isometries in the sense that the 
metric space distance between them can be bounded by some $\varepsilon$. For example,
a theorem, due to Hyers and Ulam, states the following
\begin{restatable}{theorem}{ulamstability}{\normalfont (\citet{Hyers_Ulam_1945})}\label{theorem:ulamstability}
  Let $E$ be a complete real Hilbert space. Let $\varepsilon > 0$ and $T$ be a surjection
  of E into itself that is an $\varepsilon$-isometry, that is,
  $|\rho(T(x), T(y)) - \rho(x, y)| < \varepsilon$, for all $x, y \in E$,
  where $\rho$ denotes the inner product in $E$. Assume that $T(0) = 0$. Then the limit 
  \[
      I(x)= \lim_{n \to \infty} \frac{T(2^n x)}{2^n}
  \]
  exists for every $x \in E$ and the transformation $I$ is a surjective isometry of $E$
  into itself, which satisfies $\|T(x) - I(x)\| < 10 \varepsilon,\ \forall x \in E$.
\end{restatable}

This proposition demonstrates, given any $\varepsilon > 0$ and $\varepsilon$-almost isometry, $f_{\varepsilon}$, 
the existence of an isometry, $f$, whose distance from $f_{\varepsilon}$ is at most $10 \varepsilon$, for any input $x$. 

This result spurred subsequent research, and a later bound due 
to Fickett tightened the inequality. We state his theorem here as well.
\begin{theorem}[\citet{Fickett1982}]\label{theorem:fickett}
    For a fixed integer $n \geq 2$, let $D$ be a bounded subset of $\R^n$ and let $\varepsilon > 0$
    be given. If a function $f: D \to \R^n$ satisfies
    \[
        \left| \|f(x) - f(y)\| - \|x - y\| \right| \leq \varepsilon
    \]
    for all $x, y \in D$, that is, $f$ is an $\varepsilon$-isometry of $D$, then there exists 
    an isometry $U : D \to \R^n$ such that
    \[
        \|f(x) - U(x)\| \leq 27\varepsilon^{1/2^n}
    \]
\end{theorem}
Taken together, these two results show that no matter what $\varepsilon$-almost isometry
we define, it is never ``far off'' from a full isometry, with the 
distance between the two bounded above by $27 \varepsilon^{1/2^n}$. Most recently,
\citet{vaisala} proved an even tighter bound, but its discussion is beyond the scope of this paper.

To apply Theorem~\ref{theorem:ulamstability} 
in the context of machine learning,
note that by the Nash Embedding Theorem \citep{nashembeddingthm}, we can smoothly and isometrically embed any 
Riemannian manifold $(M, g)$ into $\R^n$ for some $n$. If $M$ is compact, then the embedding of $M$ in $\R^n$
will be a compact, and therefore bounded, subset of $\R^n$. We can then apply 
Theorem~\ref{theorem:ulamstability} to any
$\varepsilon$-isometry of $M$ to get a nearby isometry of $M$ as a subset of $\R^n$. 

If $M$ is not compact, let $S \subseteq \R^n$ be its smooth isometric embedding. 
We can then apply Theorem 11.4 of \citet{wells1975}, 
which states that for a finite-dimensional Hilbert space, $H$, we can extend any  
isometry of $S$ to an isometry on the linear span of $S$. Assuming the completion, $\bar{S}$, of $S$
is contained in the linear span of $S$, we can then, for any surjective $\varepsilon$-isometry of $\bar{S}$ 
into itself, apply Theorem~\ref{theorem:ulamstability} to recover an isometry of $\bar{S}$.

\section{Method}
\subsection{Almost Equivariant Models}
Having established the theory, we now give a practical method for encoding almost equivariance in machine learning models by appealing to the Lie algebra,
$\mathfrak{g}$, of the Lie group, $G$.

\begin{definition}
  Given a connected Lie group, $G$, its Lie algebra, $\mathfrak{g}$, vector spaces, $V$ and $W$, and
  representations, $\rho_G : G \to GL(V)$ and $\rho_\mathfrak{g} : \mathfrak{g} \to \mathfrak{gl}(W)$,
  we say a model $f: V \to W$ is {\it $\varepsilon$-almost equivariant\/}
  with respect to the action of a Lie group, $G$, if 
  \[
      \|f(\rho_G(g)v) - \Phi(\rho_\mathfrak{g}(x))f(v)\| \leq \varepsilon
  \]
  for $g \in G$, $x \in \mathfrak{g}$, $v \in V$, and some $\Phi: \mathfrak{gl}(W) \to GL(W)$.
\end{definition}
Note that our definition naturally encompasses full equivariance with respect to the action of 
connected, compact Lie groups, for which the $\exp$ map is surjective, and which occurs when we 
take $\varepsilon = 0$ and define $\Phi := \exp$.


Our definition makes clear the correspondence between $G$ and the linear approximation at the identity, $e \in G$, afforded by the Lie 
algebra, $\mathfrak{g}$. Because $\rho_G(g)$ acts by $g$ on $v \in V$, we expect that there exists an element $x \in \mathfrak{g}$ such 
that the action of $\Phi(\rho_\mathfrak{g}(x))$ on $f(v) \in W$ approximates the action of some representation of $g$ on $f(v)$. 
We givea visualization of the intuition behind the definition in Figure~\ref{fig:intuition} for the case where $G = SO_2(\mathbb{C}) = S^1 \subset \mathbb{C}$.

\begin{table*}[]
\centering
\resizebox{\textwidth}{!}{%
\begin{tabular}{@{}cccccc@{}}
\toprule
Group &
Num Samples &
Model &
\begin{tabular}[c]{@{}c@{}}Rot-MNIST \\ Classification Accuracy\end{tabular} &
\begin{tabular}[c]{@{}c@{}}Pendulum \\ Regression Error (RMSE)\end{tabular} &
\begin{tabular}[c]{@{}c@{}}Pendulum \\ Average RMSE\end{tabular} \\ \midrule
SE(2) &
10 &
Almost Equivariant G-CNN &
$\mathbf{92.05 \pm 0.27}$ &
$\mathbf{0.0363 \pm 0.0004}$ &
$\mathbf{0.5571 \pm 2.1730}$ \\
\multirow{3}{*}{E(2)} &
\multirow{2}{*}{10} &
E2CNN &
{\bf \color{gray} $91.91 \pm 0.22$} &
{\bf \color{gray}$0.0349 \pm 0.0001$} &
{\bf \color{gray}$3.5987 \pm 2.8203$} \\
&
 &
Residual Pathway Prior &
$85.20 \pm 0.66$&
$0.0350 \pm 0.0001$ &
$14.4018 \pm 26.8171$ \\
&
N/A &
Approximately Equivariant G-CNN &
$84.99 \pm 0.37$ &
$0.1349 \pm 0.1414$ &
$1.4893 \pm 1.8695$ \\
T(2) &
N/A &
Standard CNN &
$85.95 \pm 0.49$ &
$0.0354 \pm 0.0009$ &
$0.6573 \pm 1.0565$ \\ \bottomrule
\end{tabular}%
}
\caption{Rot-MNIST classification accuracies and RMSE prediction errors for pendulum trajectory prediction. 
Best results are bold-faced and second-best are colored gray.}
\label{tab:rot-results}
\end{table*}

\subsection{Lie Algebra Convolutions}
We build an almost equivariant neural network layer based on the Lie algebra, $\mathfrak{g}$, of a 
matrix Lie group, $G \subseteq GL_n(\mathbb{R})$.
To parameterize our kernel function, we encode the Lie
algebra basis explicitly. For most matrix Lie groups, the corresponding Lie algebra basis
has an easily calculated set of generators, i.e.\ a set of basis elements $\{x_i\}$. 
Second, instead of mapping elements of $\mathfrak{g}$ directly to $G$ via the exponential
map, we train a neural network, $\mathcal{N}_{\theta}: \mathfrak{g} \to \mathbb{R}^{n \times n}$,
to learn an approximation to this mapping directly from data. In our experiments, each
Lie algebra convolutional layer of the network is equipped with its own $\mathcal{N}$, which 
is parameterized as an MLP with a single linear layer followed by a non-linear activation,
either a ReLU or a sigmoid function. Our method confers some key benefits over previous approaches. 
For one, the kernels used in some past works are still constrained to take as input only group elements, 
$u, x \in G$, which to some extent limits the flexibility with which they 
can model partial equivariances. In contrast, our kernel can take any 
$u, x \in \mathbb{R}^{n \times n}$ as an input, allowing us to model a more flexible class 
of functions while still maintaining the interpretability achieved by parameterizing
this function class via elements of the Lie algebra.

\begin{definition}
  We construct an {\it almost equivariant Lie algebra convolution\/}, abbreviated $\mathfrak{g}$-conv, by letting 
  $u, x = \sum_{i=1}^{\dim \mathfrak{g}}c_i x_i \in \mathfrak{g}$ and defining
  \[ 
      (k_{\omega} \star f)(u) = \int_{x \in \mathfrak{g}}
      k_{\omega}\left(\mathcal{N}_{\theta}{(x)}^{-1}\exp(u)\right)f(x)d\mu(x)
  \]
\end{definition}
Here, instead of integrating with respect to the Haar measure, we instead integrate with 
respect to the Lebesgue measure, $\mu$, defined on $\mathbb{R}^{n \times n}$. This is possible because
we are integrating over the Lie algebra, $\mathfrak{g}$, which is a vector subspace of 
$\mathbb{R}^{n \times n}$. Existing works require integrating with respect to the Haar measure 
because it is finite for compact groups, which allows one to more easily do MCMC sampling. 
Compactness is also necessary to define fully-equivariant group convolutions parameterized 
in the Lie algebra, because such a parameterization relies on the exponential map being surjective. Furthermore, 
while \citet{MacDonald_2022_CVPR} define a method for sampling from the Lie group, $G$, that 
allows the group convolution to retain full equivariance,
even for non-compact groups, by using a measure induced on the Lie algebra by the Haar
measure, we adopt our simpler approach since we are not aiming for full group equivariance 
and instead only for almost equivariance. Thus, we use a uniform measure on the Lie algebra, which
for the groups studied here amounts to the Lebesgue measure on $\R^{n\times n}$. While we still 
ultimately convolve with group elements (in the case of compact groups, for which 
$\exp: \mathfrak{g} \to G$ is surjective), our inputs, $u$, are taken from the Lie algebra, $\mathfrak{g}$,
and then pushed onto the Lie group, $G$, via the $\exp$ map.

Additionally, because the $\exp$ map is surjective only for compact Lie groups \citep{Hall2015}, the approach of parameterizing
Lie group elements by applying the $\exp$ map to elements of the Lie algebra only works 
in the compact case. Because we model the mapping function $\mathcal{N}_{\theta} : \mathfrak{g} \to G$
using a neural network, our approach extends to non-compact Lie groups.

\begin{table*}[]
  \centering
  \resizebox{\textwidth}{!}{%
  \begin{tabular}{@{}ccclccc@{}}
  \toprule
  \multirow{2}{*}{Group} & \multirow{2}{*}{Num Samples} & \multirow{2}{*}{Model}                    & \multicolumn{2}{c}{Jet Flow (RMSE)}              & \multicolumn{2}{c}{Smoke Plume (RMSE)} \\ \cmidrule(l){4-7} 
                         &                              &                                           & \multicolumn{1}{c}{Future} & Domain              & Future             & Domain            \\ \midrule
  SE(2)                  & 10                           & Almost Equivariant G-CNN                  &              $0.1931 \pm 0.0012$                 & {\color{gray}$0.2078 \pm 0.0008$} & 1.18               & 0.78              \\
  \multirow{3}{*}{E(2)}  & 10                           & E2CNN                                     & {\color{gray}$0.1919 \pm 0.0016$}                &              $0.2131 \pm 0.0023$  & 1.05               & {\color{gray} 0.76}              \\
                         & 10                           & Residual Pathway Prior                    &              $0.1947 \pm 0.0066$                 &              $0.2143 \pm 0.0057$  & {\color{gray}0.96}               & 0.83                          \\
                         & 4                            & Steerable Approximately Equivariant G-CNN &              $\mathbf{0.1597 \pm 0.0016}$        &              $\mathbf{0.1785 \pm 0.0023}$     & {\bf 0.80}               & {\bf 0.67}              \\
  T(2)                   & N/A                          & Standard CNN                              &              $0.2109 \pm 0.0068$                 &              $0.2218 \pm 0.0008$  & 1.21               & 1.10              \\ \bottomrule
  \end{tabular}%
  }
  \caption{Prediction RMSE on simulated smoke plume velocity fields and jet flow 2D turbulent velocity fields with almost rotational symmetry. The results for the baseline methods are taken from \citet{wang2022approximately} and compared against our \textit{Almost Equivariant G-CNN}. As stated in \citet{wang2022approximately}, \textbf{Future} prediction involves testing on data that lies in the future of the training data. \textbf {Domain} prediction involves training and test data that are from different spatial domains. Best results are bold-faced and second-best are colored gray.}
  \label{tab:smoke-results}
\end{table*}

\begin{theorem}
Let $G$ be a compact matrix Lie group, $G \leqslant GL_n(\R)$, and 
$\mathcal{N}_W(x) := \sigma(Wx + b)$ so that $\mathcal{N}_W^{-1}(x) = W^{-1}(\sigma^{-1}(x) - b)$.
Then the Lie algebra convolution
\[
  \int_{x \in \mathfrak{g}} k_{\omega}\left(\mathcal{N}_{W}{(x)}^{-1}\exp(u)\right)f(x)d\mu(x)
\]
is $\varepsilon$-almost equivariant.
\end{theorem}
We provide the proof in the appendix.



\section{Results}
We test our {\it Almost Equivariant G-CNN}  on a suite of tasks that span the gamut of
full and almost equivariance. For each task, we compare the performance of our model with that 
of the {\it Residual Pathway Prior} model given in \citet{finzi2021residual}, 
the {\it Approximately Equivariant G-CNN} defined in \citet{wang2022approximately}, 
the $E(2)$-equivariant and steerable {\it E2CNN} of \citet{e2cnn}, and a {\it Standard CNN} that is equivariant only to 
translations of the inputs.

\subsection{Image Classification}
We first test our model on an image classification task. We focus on the Rot-MNIST dataset,
which consists of images taken from the MNIST dataset and
subjected to random rotations. We would expect rotational equivariance to be beneficial for 
classifying these images. The training, validation, and test sets contain 10,000, 2,000, and 
50,000 images, respectively. We summarize our results on this task in Table~\ref{tab:rot-results}.
We perform a comprehensive hyperparameter grid search during training, and find that our best-performing 
model outperforms all baselines that we tested against. 
It also outperforms the standard CNN and is
marginally outperformed only by the fully-$E(2)$-equivariant E2CNN. We didn't perform any optimization of the kernel 
functions for any of the models, nor of the neural 
network mapping from the Lie Algebra to the Lie Group for our model, and expect
that with further hyperparameter tuning as well as deeper models and more complex kernel functions, we could achieve 
even higher performance(s) on the test set. 
We provide further details on the model training process in the appendix.

\subsection{Damped Pendulum}
The second task is to predict the $xy$-position, $(x, y) \in \R^2$, at time $t \in \mathbb{R}^+$ 
of a pendulum undergoing simple harmonic motion 
and subjected to wind resistance. The pendulum is modeled as a mass, $m$, connected to 
a massless rod of length $L$ subjected to an acceleration due to gravity of 
$g = -9.8 \text{m}/\text{sec}^2$ and position function $\theta(t)$. 
The differential equation governing this motion is
\[
  \frac{\partial^2 \theta}{\partial t^2} + \frac{\lambda}{m}\frac{\partial \theta}{\partial t} + \frac{g}{L}\theta = 0
\]
where $\lambda$ is the coefficient of friction governing the wind resistance which 
is modeled as a force 
\[
  F_w = -\lambda L \frac{\partial \theta}{\partial t}
\]
We simulate the trajectory of the pendulum using the Runge-Kutta method to obtain an 
iterative, approximate solution to the above, second-order differential equation. We 
sample $\theta(t)$ for 6000 values of $t \in (0, 60)$ using a $dt = 0.01$ and setting 
$m = L = 1$, $\theta(0) = \pi/3$, $\frac{\partial \theta}{\partial t}(0) = 0$, and 
$\lambda = 0.2$. We partition this data into a 90\%/10\% train-test split and train 
a series of models to predict $xy$-position from the time $t \in (0, 60)$. Because the pendulum 
rotates about a vertical line, we again expect that rotational equivariance would be beneficial for this task.

Table~\ref{tab:rot-results} summarizes our results. We find that our Almost Equivariant G-CNN, the 
E2CNN, the Approximately Equivariant G-CNN, and the Residual Pathway Prior all achieve nearly identical performance, 
slightly beating out the standard CNN, which has many more parameters than the other baselines.
Relative to the E2CNN and the RPP models, our model achieves significantly lower mean RMSE across hyperparameter configurations.
The RPP model, in particular, demonstrates a high sensitivity to hyperparameter settings. Our model uses far fewer parameters 
than the standard CNN and a number of parameters comparable to the other baselines. While our best-performing model uses a 
kernel size of 4 compared to a kernel size of 2 used for the CNN, it uses only 1 hidden layer and 16 hidden channels, compared to 
the CNN which uses 3 hidden layers having hidden channel sizes of 32, 64, and 128, respectively.

\subsection{Smoke Plume}
Next, we test our model on an almost equivariant prediction task. The dataset we use is the 
smoke plume dataset of \citet{wang2022approximately} consisting of $64 \times 64$ 2D velocity vector fields 
of smoke simulations with different initial conditions and external forces, all generated using 
the PDE simulation framework, PhiFlow \citep{Holl2020Learning}. Specifically, we use the subset 
of the data that features rotational almost equivariance. As stated in \citet{wang2022approximately},
``both the inflow location and the direction of the buoyant forces possess a perfect rotation symmetry 
with respect to the $C_4$ group, but the buoyancy factor varies with the inflow positions to break 
the rotational symmetry.'' All models are trained to predict the raw velocity fields at the next time step 
given the raw velocity fields at the previous timestep as input. 

Due to computational constraints, we only run our method on this data and compare to the baseline 
results reported in \citet{wang2022approximately}. 
Table~\ref{tab:smoke-results} shows how our method compares 
to the baselines. Due to computational constraints, we were 
unable to run a full hyperparameter sweep and suspect that 
doing so would boost our model's performance even further.

\subsection{Jet Flow}
Finally, we test on one more almost equivariant dataset.
As described in \citet{wang2022approximately}, this dataset
contains samples of 2D turbulent velocity fields taken 
from NASA multi-stream jets that were measured using time-resolved 
particle image velocimetry as described in \citet{bridges}. 
We follow the procedure described in 
\citet{wang2022approximately}, and ``train and test
on twenty-four $62 \times 23$ sub-regions of jet flows.''
Table~\ref{tab:smoke-results} shows our results.

\section{Discussion}
In this work, we proposed a definition of almost equivariance that encompassed previous
definitions of full and approximate/partial/soft equivariance. We connected this definition 
to mathematical theory by showing that, given an abelian isometry group, $G$, acting on a Riemannian manifold,
$M$, then any isometry, $f$ of $M$, is equivariant to the action of $G$, and furthermore that there exists an 
$\varepsilon$-almost isometry, $f_{\varepsilon}$ of $M$, not more than $\varepsilon$ from $f$ in the supremum norm,
such that $f_{\varepsilon}$ is almost equivariant to the action of $G$.
Next, we showed that nothing is lost by taking $f$ and $f_{\varepsilon}$ to be isometric and almost isometric embeddings, respectively, 
of $M$ into $\R^n$. We then appealed to Ulam Stability Theory 
to give conditions under which we 
can get an isometry of a complete, real Hilbert space close to an almost isometry of the same space.
All of this taken together demonstrates that there exist almost equivariant functions that are never ``far'' from fully equivariant functions, 
given some constraints on the group action and class of functions, in a sense that can be mathematically quantified. 

We next introduced a convolution on the elements of a Lie algebra that approximates a
fully equivariant group convolution. We then showed that such a convolution can model almost
equivariance relative to {\it any\/} group action, even those of non-compact groups.
We validated our assumptions by testing our model on a 2D image classification task,
a 1D sequence regression task, and a 2D sequence regression task. On all tasks, our model exceeded
or met the performance of state-of-the-art equivariant and almost equivariant baseline models.
This demonstrates the utility of our method across a variety of scientific domains and prediction task types.

\section{Future Work}
One line of future work will involve testing our model architecture on a wider class of group actions. 
While our model is general enough to handle the action of 
any group, including those of non-compact groups, we have not yet tested it on groups aside from $E(2)$.
\citet{lawrence2023} points to some potential
applications of equivariance to non-compact Lie groups. 

Next, there exist a number of ways to further expound upon the theoretical results given here.
One potential angle to consider is whether variations of Theorem~\ref{theorem:ulamstability} 
can be made to hold for arbitrary Riemannian manifolds and not just Hilbert and Euclidean spaces,
respectively. Another direction would involve undertaking a rigorous analysis of the conditions under
which almost equivariance to the action of a non-abelian group can be imposed upon a function. 
We here gave
proof of the existence of almost isometries of Riemannian manifolds that are almost equivariant to certain abelian group actions,
which we believe to be the most useful direction as, in practice, one normally seeks to take a fully equivariant 
model and make it almost equivariant. 
That said, the more difficult mathematical question is to consider when, given an almost equivariant
function on a manifold, it can be transformed into a fully equivariant function on the same manifold. We leave this direction
for future work.

Finally, it is known that fully-equivariant kernel sharing for G-CNNs requires that the group act transitively on the input space \citep{cicn}.
An interesting direction for future work would be investigating the extent to which this assumption is required for almost 
equivariant kernel sharing.

\section{Acknowledgements}
We thank Frederic Sala, Jason Hartford, and Andrew Zimmer for their valuable feedback on this work.

\nocite{*}
\bibliographystyle{plainnat}
\bibliography{references}

\begin{thebibliography}{51}
\providecommand{\natexlab}[1]{#1}
\providecommand{\url}[1]{\texttt{#1}}
\expandafter\ifx\csname urlstyle\endcsname\relax
  \providecommand{\doi}[1]{doi: #1}\else
  \providecommand{\doi}{doi: \begingroup \urlstyle{rm}\Url}\fi

\bibitem[Batzner et~al.(2022)Batzner, Musaelian, Sun, Geiger, Mailoa,
  Kornbluth, Molinari, Smidt, and Kozinsky]{nequip}
Simon Batzner, Albert Musaelian, Lixin Sun, Mario Geiger, Jonathan~P. Mailoa,
  Mordechai Kornbluth, Nicola Molinari, Tess~E. Smidt, and Boris Kozinsky.
\newblock E(3)-equivariant graph neural networks for data-efficient and
  accurate interatomic potentials.
\newblock \emph{Nature Communications}, 13\penalty0 (1):\penalty0 2453, 2022.
\newblock \doi{10.1038/s41467-022-29939-5}.
\newblock URL \url{https://doi.org/10.1038/s41467-022-29939-5}.

\bibitem[Bridges and Wernet(2017)]{bridges}
James Bridges and Mark Wernet.
\newblock Measurements of turbulent convection speeds in multistream jets using
  time-resolved piv.
\newblock 06 2017.
\newblock \doi{10.2514/6.2017-4041}.

\bibitem[Bronstein et~al.(2021)Bronstein, Bruna, Cohen, and
  Veličković]{bronstein2021geometric}
Michael~M. Bronstein, Joan Bruna, Taco Cohen, and Petar Veličković.
\newblock Geometric deep learning: Grids, groups, graphs, geodesics, and
  gauges, 2021.

\bibitem[Cohen and Welling(2016)]{cohenc16}
Taco Cohen and Max Welling.
\newblock Group equivariant convolutional networks.
\newblock In Maria~Florina Balcan and Kilian~Q. Weinberger, editors,
  \emph{Proceedings of The 33rd International Conference on Machine Learning},
  volume~48 of \emph{Proceedings of Machine Learning Research}, pages
  2990--2999, New York, New York, USA, 20--22 Jun 2016. PMLR.
\newblock URL \url{https://proceedings.mlr.press/v48/cohenc16.html}.

\bibitem[Cohen et~al.(2019)Cohen, Geiger, and Weiler]{cohen2020general}
Taco~S Cohen, Mario Geiger, and Maurice Weiler.
\newblock A general theory of equivariant cnns on homogeneous spaces.
\newblock In H.~Wallach, H.~Larochelle, A.~Beygelzimer, F.~d\textquotesingle
  Alch\'{e}-Buc, E.~Fox, and R.~Garnett, editors, \emph{Advances in Neural
  Information Processing Systems}, volume~32. Curran Associates, Inc., 2019.
\newblock URL
  \url{https://proceedings.neurips.cc/paper_files/paper/2019/file/b9cfe8b6042cf759dc4c0cccb27a6737-Paper.pdf}.

\bibitem[Dehmamy et~al.(2021)Dehmamy, Walters, Liu, Wang, and Yu]{nimaalgebra}
Nima Dehmamy, Robin Walters, Yanchen Liu, Dashun Wang, and Rose Yu.
\newblock Automatic symmetry discovery with lie algebra convolutional network.
\newblock In M.~Ranzato, A.~Beygelzimer, Y.~Dauphin, P.S. Liang, and J.~Wortman
  Vaughan, editors, \emph{Advances in Neural Information Processing Systems},
  volume~34, pages 2503--2515. Curran Associates, Inc., 2021.
\newblock URL
  \url{https://proceedings.neurips.cc/paper_files/paper/2021/file/148148d62be67e0916a833931bd32b26-Paper.pdf}.

\bibitem[Eade(2017)]{Eade2017}
Ethan Eade.
\newblock Lie groups for 2d and 3d transformations, May 2017.
\newblock URL \url{https://ethaneade.com/lie.pdf}.
\newblock Date Accessed: 2023-06-14.

\bibitem[Etingof et~al.(2011)Etingof, Golberg, Hensel, Liu, Schwendner,
  Vaintrob, and Yudovina]{etingof2011representation}
Pavel Etingof, Oleg Golberg, Sebastian Hensel, Tiankai Liu, Alex Schwendner,
  Dmitry Vaintrob, and Elena Yudovina.
\newblock Introduction to representation theory, 2011.

\bibitem[Fickett(1982)]{Fickett1982}
James Fickett.
\newblock Approximate isometries on bounded sets with an application to measure
  theory.
\newblock \emph{Studia Mathematica}, 72\penalty0 (1):\penalty0 37--46, 1982.
\newblock URL \url{http://eudml.org/doc/218431}.

\bibitem[Finzi et~al.(2020)Finzi, Stanton, Izmailov, and
  Wilson]{finzi2020generalizing}
Marc Finzi, Samuel Stanton, Pavel Izmailov, and Andrew~Gordon Wilson.
\newblock Generalizing convolutional neural networks for equivariance to lie
  groups on arbitrary continuous data.
\newblock In Hal~Daumé III and Aarti Singh, editors, \emph{Proceedings of the
  37th International Conference on Machine Learning}, volume 119 of
  \emph{Proceedings of Machine Learning Research}, pages 3165--3176. PMLR,
  13--18 Jul 2020.
\newblock URL \url{https://proceedings.mlr.press/v119/finzi20a.html}.

\bibitem[Finzi et~al.(2021{\natexlab{a}})Finzi, Welling, and Wilson]{finzi21a}
Marc Finzi, Max Welling, and Andrew Gordon~Gordon Wilson.
\newblock A practical method for constructing equivariant multilayer
  perceptrons for arbitrary matrix groups.
\newblock In Marina Meila and Tong Zhang, editors, \emph{Proceedings of the
  38th International Conference on Machine Learning}, volume 139 of
  \emph{Proceedings of Machine Learning Research}, pages 3318--3328. PMLR,
  18--24 Jul 2021{\natexlab{a}}.
\newblock URL \url{https://proceedings.mlr.press/v139/finzi21a.html}.

\bibitem[Finzi et~al.(2021{\natexlab{b}})Finzi, Benton, and
  Wilson]{finzi2021residual}
Marc~Anton Finzi, Gregory Benton, and Andrew~Gordon Wilson.
\newblock Residual pathway priors for soft equivariance constraints.
\newblock In A.~Beygelzimer, Y.~Dauphin, P.~Liang, and J.~Wortman Vaughan,
  editors, \emph{Advances in Neural Information Processing Systems},
  2021{\natexlab{b}}.
\newblock URL \url{https://openreview.net/forum?id=k505ekjMzww}.

\bibitem[Fulton and Harris(2004)]{Fulton_Harris_2004}
William Fulton and Joe Harris.
\newblock \emph{Representation theory: A first course}.
\newblock Springer, 2004.

\bibitem[Gerken et~al.(2022{\natexlab{a}})Gerken, Carlsson, Linander, Ohlsson,
  Petersson, and Persson]{eqvsaug}
Jan Gerken, Oscar Carlsson, Hampus Linander, Fredrik Ohlsson, Christoffer
  Petersson, and Daniel Persson.
\newblock Equivariance versus augmentation for spherical images.
\newblock In Kamalika Chaudhuri, Stefanie Jegelka, Le~Song, Csaba Szepesvari,
  Gang Niu, and Sivan Sabato, editors, \emph{Proceedings of the 39th
  International Conference on Machine Learning}, volume 162 of
  \emph{Proceedings of Machine Learning Research}, pages 7404--7421. PMLR,
  17--23 Jul 2022{\natexlab{a}}.
\newblock URL \url{https://proceedings.mlr.press/v162/gerken22a.html}.

\bibitem[Gerken et~al.(2022{\natexlab{b}})Gerken, Carlsson, Linander, Ohlsson,
  Petersson, and Persson]{gerken22a}
Jan Gerken, Oscar Carlsson, Hampus Linander, Fredrik Ohlsson, Christoffer
  Petersson, and Daniel Persson.
\newblock Equivariance versus augmentation for spherical images.
\newblock In Kamalika Chaudhuri, Stefanie Jegelka, Le~Song, Csaba Szepesvari,
  Gang Niu, and Sivan Sabato, editors, \emph{Proceedings of the 39th
  International Conference on Machine Learning}, volume 162 of
  \emph{Proceedings of Machine Learning Research}, pages 7404--7421. PMLR,
  17--23 Jul 2022{\natexlab{b}}.
\newblock URL \url{https://proceedings.mlr.press/v162/gerken22a.html}.

\bibitem[Gruver et~al.(2023)Gruver, Finzi, Goldblum, and Wilson]{gruver2023the}
Nate Gruver, Marc~Anton Finzi, Micah Goldblum, and Andrew~Gordon Wilson.
\newblock The lie derivative for measuring learned equivariance.
\newblock In \emph{The Eleventh International Conference on Learning
  Representations}, 2023.
\newblock URL \url{https://openreview.net/forum?id=JL7Va5Vy15J}.

\bibitem[Gu(2017)]{Gu_2017}
Chenlin Gu.
\newblock Lecture notes on metric space and gromov-hausdorff distance, Sep
  2017.
\newblock URL \url{https://chenlin-gu.github.io/notes/GromovHausdorff.pdf}.

\bibitem[Hall(2015)]{Hall2015}
Brian Hall.
\newblock \emph{Lie Groups, Lie Algebras, and Representations: An Elementary
  Introduction}.
\newblock Springer International Publishing, Cham, 2015.
\newblock ISBN 978-3-319-13467-3.
\newblock \doi{10.1007/978-3-319-13467-3_1}.
\newblock URL \url{https://doi.org/10.1007/978-3-319-13467-3}.

\bibitem[Holl et~al.(2020)Holl, Thuerey, and Koltun]{Holl2020Learning}
Philipp Holl, Nils Thuerey, and Vladlen Koltun.
\newblock Learning to control pdes with differentiable physics.
\newblock In \emph{International Conference on Learning Representations}, 2020.
\newblock URL \url{https://openreview.net/forum?id=HyeSin4FPB}.

\bibitem[Hutchinson et~al.(2021{\natexlab{a}})Hutchinson, Lan, Zaidi, Dupont,
  Teh, and Kim]{2021lietransformer}
Michael~J Hutchinson, Charline~Le Lan, Sheheryar Zaidi, Emilien Dupont,
  Yee~Whye Teh, and Hyunjik Kim.
\newblock Lietransformer: Equivariant self-attention for lie groups.
\newblock In Marina Meila and Tong Zhang, editors, \emph{Proceedings of the
  38th International Conference on Machine Learning}, volume 139 of
  \emph{Proceedings of Machine Learning Research}, pages 4533--4543. PMLR,
  18--24 Jul 2021{\natexlab{a}}.
\newblock URL \url{https://proceedings.mlr.press/v139/hutchinson21a.html}.

\bibitem[Hutchinson et~al.(2021{\natexlab{b}})Hutchinson, Lan, Zaidi, Dupont,
  Teh, and Kim]{lietransformer}
Michael~J Hutchinson, Charline~Le Lan, Sheheryar Zaidi, Emilien Dupont,
  Yee~Whye Teh, and Hyunjik Kim.
\newblock Lietransformer: Equivariant self-attention for lie groups.
\newblock In Marina Meila and Tong Zhang, editors, \emph{Proceedings of the
  38th International Conference on Machine Learning}, volume 139 of
  \emph{Proceedings of Machine Learning Research}, pages 4533--4543. PMLR,
  18--24 Jul 2021{\natexlab{b}}.
\newblock URL \url{https://proceedings.mlr.press/v139/hutchinson21a.html}.

\bibitem[Hyers and Ulam(1945)]{Hyers_Ulam_1945}
D.~H. Hyers and S.~M. Ulam.
\newblock On approximate isometries.
\newblock \emph{Bulletin of the American Mathematical Society}, 51\penalty0
  (4):\penalty0 288–292, 1945.
\newblock \doi{10.1090/s0002-9904-1945-08337-2}.

\bibitem[Ivanov(1997)]{ivanov1997gromov}
SV~Ivanov.
\newblock Gromov--hausdorff convergence and volumes of manifolds.
\newblock \emph{Algebra i Analiz}, 9\penalty0 (5):\penalty0 65--83, 1997.

\bibitem[Kondor(2008)]{10.5555/1570977}
Imre~Risi Kondor.
\newblock \emph{Group Theoretical Methods in Machine Learning}.
\newblock PhD thesis, USA, 2008.
\newblock AAI3333377.

\bibitem[Kondor and Trivedi(2018)]{kondor18a}
Risi Kondor and Shubhendu Trivedi.
\newblock On the generalization of equivariance and convolution in neural
  networks to the action of compact groups.
\newblock In Jennifer Dy and Andreas Krause, editors, \emph{Proceedings of the
  35th International Conference on Machine Learning}, volume~80 of
  \emph{Proceedings of Machine Learning Research}, pages 2747--2755. PMLR,
  10--15 Jul 2018.
\newblock URL \url{https://proceedings.mlr.press/v80/kondor18a.html}.

\bibitem[Lafarge et~al.(2020)Lafarge, Bekkers, Pluim, Duits, and Veta]{lafarge}
Maxime~W. Lafarge, Erik~J. Bekkers, Josien P.~W. Pluim, Remco Duits, and Mitko
  Veta.
\newblock Roto-translation equivariant convolutional networks: Application to
  histopathology image analysis.
\newblock \emph{CoRR}, abs/2002.08725, 2020.
\newblock URL \url{https://arxiv.org/abs/2002.08725}.

\bibitem[Lawrence and Harris(2023)]{lawrence2023}
Hannah Lawrence and Mitchell~Tong Harris.
\newblock Learning polynomial problems with sl(2)-equivariance, 2023.
\newblock URL \url{https://openreview.net/pdf?id=mRr53KWuf1}.

\bibitem[Lee(2003)]{Lee2003}
John~M. Lee.
\newblock \emph{Introduction to Smooth Manifolds}.
\newblock Springer New York, New York, NY, 2003.
\newblock ISBN 978-0-387-21752-9.
\newblock \doi{10.1007/978-0-387-21752-9}.
\newblock URL \url{https://doi.org/10.1007/978-0-387-21752-9}.

\bibitem[Lee(2018)]{Lee2018}
John~M. Lee.
\newblock \emph{Introduction to Riemannian Manifolds}.
\newblock Springer International Publishing, Cham, 2018.
\newblock ISBN 978-3-319-91755-9.
\newblock \doi{10.1007/978-3-319-91755-9}.
\newblock URL \url{https://doi.org/10.1007/978-3-319-91755-9}.

\bibitem[Liao and Smidt(2023)]{liao2023equiformer}
Yi-Lun Liao and Tess Smidt.
\newblock Equiformer: Equivariant graph attention transformer for 3d atomistic
  graphs.
\newblock In \emph{International Conference on Learning Representations}, 2023.
\newblock URL \url{https://openreview.net/forum?id=KwmPfARgOTD}.

\bibitem[MacDonald et~al.(2022)MacDonald, Ramasinghe, and
  Lucey]{MacDonald_2022_CVPR}
Lachlan~E. MacDonald, Sameera Ramasinghe, and Simon Lucey.
\newblock Enabling equivariance for arbitrary lie groups.
\newblock In \emph{Proceedings of the IEEE/CVF Conference on Computer Vision
  and Pattern Recognition (CVPR)}, pages 8183--8192, June 2022.

\bibitem[Matteo(2021)]{OliviaDiMatteo2021}
Olivia~Di Matteo.
\newblock Understanding the haar measure.
\newblock \url{https://pennylane.ai/qml/demos/tutorial_haar_measure}, 02 2021.
\newblock Date Accessed: 2023-06-14.

\bibitem[Moore and Schlafly(1980)]{equivariantnash}
John~Douglas Moore and Roger Schlafly.
\newblock On equivariant isometric embeddings.
\newblock \emph{Mathematische Zeitschrift}, 173\penalty0 (2):\penalty0
  119--133, 1980.
\newblock \doi{10.1007/BF01159954}.
\newblock URL \url{https://doi.org/10.1007/BF01159954}.

\bibitem[Myers and Steenrod(1939)]{myersSteenrod}
S.~B. Myers and N.~E. Steenrod.
\newblock The group of isometries of a riemannian manifold.
\newblock \emph{Annals of Mathematics}, 40\penalty0 (2):\penalty0 400--416,
  1939.
\newblock ISSN 0003486X.
\newblock URL \url{http://www.jstor.org/stable/1968928}.

\bibitem[Nash(1954)]{nashembeddingthm}
John Nash.
\newblock C1 isometric imbeddings.
\newblock \emph{Annals of Mathematics}, 60\penalty0 (3):\penalty0 383--396,
  1954.
\newblock ISSN 0003486X.
\newblock URL \url{http://www.jstor.org/stable/1969840}.

\bibitem[Petrache and Trivedi(2023)]{petrache2023approximationgeneralization}
Mircea Petrache and Shubhendu Trivedi.
\newblock Approximation-generalization trade-offs under (approximate) group
  equivariance, 2023.

\bibitem[Rassias et~al.(2018)Rassias, Brzdkek, Popa, Racsa, and Xu]{hyersulam}
Themistocles~M. Rassias, Janusz Brzdkek, Dorian Popa, Ioan Racsa, and Bing Xu.
\newblock Chapter 2 - ulam stability of operators in normed spaces.
\newblock In Themistocles~M. Rassias, Janusz Brzdkek, Dorian Popa, Ioan Racsa,
  and Bing Xu, editors, \emph{Ulam Stability of Operators}, Mathematical
  Analysis and its Applications, pages 33--68. Academic Press, 2018.

\bibitem[Romero and Lohit(2022)]{romero2022learning}
David~W. Romero and Suhas Lohit.
\newblock Learning partial equivariances from data.
\newblock In S.~Koyejo, S.~Mohamed, A.~Agarwal, D.~Belgrave, K.~Cho, and A.~Oh,
  editors, \emph{Advances in Neural Information Processing Systems}, volume~35,
  pages 36466--36478. Curran Associates, Inc., 2022.
\newblock URL
  \url{https://proceedings.neurips.cc/paper_files/paper/2022/file/ec51d1fe4bbb754577da5e18eb54e6d1-Paper-Conference.pdf}.

\bibitem[Satorras et~al.(2021)Satorras, Hoogeboom, and Welling]{satorras21a}
V\'{\i}ctor~Garcia Satorras, Emiel Hoogeboom, and Max Welling.
\newblock E(n) equivariant graph neural networks.
\newblock In Marina Meila and Tong Zhang, editors, \emph{Proceedings of the
  38th International Conference on Machine Learning}, volume 139 of
  \emph{Proceedings of Machine Learning Research}, pages 9323--9332. PMLR,
  18--24 Jul 2021.
\newblock URL \url{https://proceedings.mlr.press/v139/satorras21a.html}.

\bibitem[Ulam(1960)]{Ulam1960}
Stanislaw~M. Ulam.
\newblock \emph{A collection of mathematical problems}.
\newblock Interscience Publishers, 1960.

\bibitem[V{\"a}is{\"a}l{\"a}(2002)]{vaisala}
Jussi V{\"a}is{\"a}l{\"a}.
\newblock Isometric approximation property in euclidean spaces.
\newblock \emph{Israel Journal of Mathematics}, 128\penalty0 (1):\penalty0
  1--27, 2002.
\newblock \doi{10.1007/BF02785416}.
\newblock URL \url{https://doi.org/10.1007/BF02785416}.

\bibitem[van~der Ouderaa et~al.(2022)van~der Ouderaa, Romero, and van~der
  Wilk]{ouderaa2022relaxing}
Tycho~F.A. van~der Ouderaa, David~W. Romero, and Mark van~der Wilk.
\newblock Relaxing equivariance constraints with non-stationary continuous
  filters.
\newblock In Alice~H. Oh, Alekh Agarwal, Danielle Belgrave, and Kyunghyun Cho,
  editors, \emph{Advances in Neural Information Processing Systems}, 2022.
\newblock URL \url{https://openreview.net/forum?id=5oEk8fvJxny}.

\bibitem[Wang et~al.(2023)Wang, Zhu, Park, Jia, Su, Platt, and
  Walters]{wang2023a}
Dian Wang, Xupeng Zhu, Jung~Yeon Park, Mingxi Jia, Guanang Su, Robert Platt,
  and Robin Walters.
\newblock A general theory of correct, incorrect, and extrinsic equivariance.
\newblock In \emph{Thirty-seventh Conference on Neural Information Processing
  Systems}, 2023.
\newblock URL \url{https://openreview.net/forum?id=2FMJtNDLeE}.

\bibitem[Wang et~al.(2022{\natexlab{a}})Wang, Walters, and
  Yu]{wang2022approximately}
Rui Wang, Robin Walters, and Rose Yu.
\newblock Approximately equivariant networks for imperfectly symmetric
  dynamics.
\newblock In Kamalika Chaudhuri, Stefanie Jegelka, Le~Song, Csaba Szepesvari,
  Gang Niu, and Sivan Sabato, editors, \emph{Proceedings of the 39th
  International Conference on Machine Learning}, volume 162 of
  \emph{Proceedings of Machine Learning Research}, pages 23078--23091. PMLR,
  17--23 Jul 2022{\natexlab{a}}.
\newblock URL \url{https://proceedings.mlr.press/v162/wang22aa.html}.

\bibitem[Wang et~al.(2022{\natexlab{b}})Wang, Walters, and Yu]{wang2022data}
Rui Wang, Robin Walters, and Rose Yu.
\newblock Data augmentation vs. equivariant networks: A theory of
  generalization on dynamics forecasting, 2022{\natexlab{b}}.

\bibitem[Weiler and Cesa(2019)]{e2cnn}
Maurice Weiler and Gabriele Cesa.
\newblock {General E(2)-Equivariant Steerable CNNs}.
\newblock In \emph{Conference on Neural Information Processing Systems
  (NeurIPS)}, 2019.

\bibitem[Weiler et~al.(2021)Weiler, Forr{\'{e}}, Verlinde, and Welling]{cicn}
Maurice Weiler, Patrick Forr{\'{e}}, Erik Verlinde, and Max Welling.
\newblock Coordinate independent convolutional networks - isometry and gauge
  equivariant convolutions on riemannian manifolds.
\newblock \emph{CoRR}, abs/2106.06020, 2021.
\newblock URL \url{https://arxiv.org/abs/2106.06020}.

\bibitem[Wells and Williams(1975)]{wells1975}
J.H. Wells and L.R. Williams.
\newblock \emph{The Extension Problem for Contractions and Isometries}, pages
  46--75.
\newblock Springer Berlin Heidelberg, 1975.

\bibitem[Worrall and Welling(2019)]{deepscalespaces}
Daniel Worrall and Max Welling.
\newblock Deep scale-spaces: Equivariance over scale.
\newblock In H.~Wallach, H.~Larochelle, A.~Beygelzimer, F.~d\textquotesingle
  Alch\'{e}-Buc, E.~Fox, and R.~Garnett, editors, \emph{Advances in Neural
  Information Processing Systems}, volume~32. Curran Associates, Inc., 2019.
\newblock URL
  \url{https://proceedings.neurips.cc/paper_files/paper/2019/file/f04cd7399b2b0128970efb6d20b5c551-Paper.pdf}.

\bibitem[Xu et~al.(2021)Xu, Kim, Rainforth, and Teh]{2021subsampling}
Jin Xu, Hyunjik Kim, Thomas Rainforth, and Yee Teh.
\newblock Group equivariant subsampling.
\newblock In M.~Ranzato, A.~Beygelzimer, Y.~Dauphin, P.S. Liang, and J.~Wortman
  Vaughan, editors, \emph{Advances in Neural Information Processing Systems},
  volume~34, pages 5934--5946. Curran Associates, Inc., 2021.
\newblock URL
  \url{https://proceedings.neurips.cc/paper_files/paper/2021/file/2ea6241cf767c279cf1e80a790df1885-Paper.pdf}.

\bibitem[Zhou et~al.(2022)Zhou, Shrikumar, and Kundaje]{equivariantdna}
Hannah Zhou, Avanti Shrikumar, and Anshul Kundaje.
\newblock Towards a better understanding of reverse-complement equivariance for
  deep learning models in genomics.
\newblock In David~A. Knowles, Sara Mostafavi, and Su-In Lee, editors,
  \emph{Proceedings of the 16th Machine Learning in Computational Biology
  meeting}, volume 165 of \emph{Proceedings of Machine Learning Research},
  pages 1--33. PMLR, 22--23 Nov 2022.
\newblock URL \url{https://proceedings.mlr.press/v165/zhou22a.html}.

\end{thebibliography}

\newpage
\appendix
\section{Appendix}
\subsection{Proofs of Theorems}

\subsubsection{Almost Equivariance of the Lie Algebra Convolution}
Let $G$ be a compact matrix Lie group, $G \leqslant GL_n(\R)$, and 
$\mathcal{N}_W(x) := \sigma(Wx + b)$ so that $\mathcal{N}_W^{-1}(x) = W^{-1}(\sigma^{-1}(x) - b)$.
Then the Lie algebra convolution
\[
  \int_{x \in \mathfrak{g}} k_{\omega}\left(\mathcal{N}_{W}{(x)}^{-1}\exp(u)\right)f(x)d\mu(x)
\]
is $\varepsilon$-almost equivariant.

\begin{proof}
  Take $G$ to be smoothly and isometrically embedded in $\R^n$ by the Nash Embedding Theorem, and let $\| \cdot \|_2$ 
  be the norm defined with respect to the Euclidean metric on the ambient space, $\R^n$.
  This allows us to subtract $g \in G$ and $x \in \mathfrak{g}$ as elements of the ambient $\R^n$, despite the fact
  that they live in different natural spaces.
  Let $u = I + \sum_i a_i V_i, x = I + \sum_i c_i V_i \in \mathfrak{g}$, where $I$ is the identity element
  of $G$, and let $h = \exp(u), g \in G$. Define 
  \[
      \delta = \sup_{g, h \in G} \|g - h\|_2
  \]
  Since $G$ is compact, $G$ is bounded as a subset of $\R^n$, and such an $\varepsilon$ must exist.
  Furthermore, because $G$ is compact, the exponential map $\exp: \mathfrak{g} \to G$ is surjective. 
  We write our convolution as 
  \[
      \int_{x \in \mathfrak{g}} k_{\omega}\left(\mathcal{N}_{W}{(x)}^{-1}\exp(u)\right)f(x)d\mu(x)
  \]
  In practice, we discretize this integral by drawing samples of $x_i \in \mathfrak{g}, i = 1, \ldots, N$
  to approximate a convolution of some $g_i \in G$ with $h = \exp(u) \in G$, so the MCMC approximation of the 
  above full convolution becomes
  \[
      \frac{\text{Vol(G)}}{N}\sum_{i=1}^N k_{\omega}\left(\mathcal{N}_{W}{(x_i)}^{-1}\exp(u)\right)f(x_i)
  \]
  assuming a bounded kernel function, $k_{\omega}(x) \leq K\|x\|_2$, we have
  \begin{align}
      \left\|\sum_{i=1}^N k_{\omega}(g_i^{-1}\exp(u)) - k_{\omega}(\mathcal{N}^{-1}_W(x_i)\exp(u))\right\|_2 
      &\leq \sum_{i=1}^N \left\|k_{\omega}(g_i^{-1}\exp(u)) - k_{\omega}(\mathcal{N}^{-1}_W(x_i)\exp(u))\right\|_2 \\
      &\leq \sum_{i=1}^N K\|g_i^{-1}\exp(u)\|_2 - K\|\mathcal{N}^{-1}_W(x_i)\exp(u)\|_2 \\
      &= K\sum_{i=1}^N \|g_i^{-1}\exp(u)\|_2 - \|\mathcal{N}^{-1}_W(x_i)\exp(u)\|_2 \\
      &\leq K \|\exp(u)\|_2 \sum_{i=1}^N \|g_i^{-1}\|_2 - \|\mathcal{N}^{-1}_W(x_i)\|_2
  \end{align}
  where we get (5) by the triangle inequality, (6) by the boundedness of $k_{\omega}$, (7) by factoring out $K$, and
  (8) by the Cauchy-Schwarz inequality.
  If we set
  \[
      \delta_x = \max_i \left(\|g_i^{-1}\|_2 - \|\mathcal{N}^{-1}_W(x_i)\|_2\right)
  \]
  then we can substitute back into equation (8) and get
  \[
      K\|\exp(u)\|_2\sum_{i=1}^N \|g_i^{-1}\|_2 - \|\mathcal{N}^{-1}_W(x_i)\|_2 \leq KN \delta_x \|\exp(u)\|_2
  \]
  and setting $\varepsilon = KN \delta_x \|\exp(u)\|_2$ gets us what we want.
\end{proof}
Thus, by bounding the error $\|g - \mathcal{N}^{-1}(x)\|_2 < \delta_x$ for $g \in G, x \in \mathfrak{g}$ in 
the neural network $\mathcal{N}: \mathfrak{g} \to G$ and the kernel function $k_{\omega}(x) \leq K\|x\|_2$, we can 
control the extent to which the Lie algebra convolution is $\varepsilon$-almost equivariant.

\subsubsection{Proof of Theorem 3.11}

We recall \cref{theorem:ulamstability}.
\newline
\ulamstability*
This proof, taken from \citep{Hyers_Ulam_1945}, is reproduced here for the reader's convenience.
\begin{proof}
 Put $r = \|x\|$. Then $\left| \|T(x)\| - r \right| < \varepsilon$ 
 and $\left| \|T(x) - T(2x)\| - r\right| < \varepsilon$. Put also $y_0 = T(2x)/2$,
 so that $|r - \|y_0\| | < \varepsilon / 2$. Consider the intersection of the two 
 spheres: $S_1 = [y; \|y\| < r + \varepsilon]$, $S_2 = [y; \|y - 2y_0| < r + \varepsilon]$.
 Now $T(x)$ belongs to this intersection, and for any point $y$ of $S_1 \cap S_2$ we have 
 \begin{align*}
  2\|y - y_0\|^2 &= 2\|y\|^2 + 2\|y_0\|^2 - 4(y, y_0); \\
  \|y - 2y_0\|^2 &= \|y\|^2 + 4\|y_0\|^2 - 4(y, y_0) < (r + \varepsilon)^2
 \end{align*}
 and $\|y\|^2 < (r + \varepsilon)^2$. It follows that
 \begin{align*}
      2\|y - y_0\|^2 &< (r + \varepsilon)^2 + \|y\|^2 - 2\|y_0\|^2 < 2(r + \varepsilon)^2 - 2\|y_0\|^2 \\
                     &< 2(r + \varepsilon)^2 - 2(r - \varepsilon/2)^2 = 6\varepsilon r + 3\varepsilon^2/2.
 \end{align*}
 Hence, $\|T(x) - T(2x)/2\| < 2(\varepsilon \|x\|)^{1/2}$ if $\|x\| \geq \varepsilon$, and 
 $\|T(x) - T(2x)/2\| < 2\varepsilon$ in the contrary case. Therefore, for all $x \in E$, the inequality 
 \begin{equation}
  \|T(x/2) - T(x)/2\| < 2^{-1/2}k(\|x\|)^{1/2} + 2\varepsilon
 \end{equation}
 is satisfied, where $k = 2\varepsilon^{1/2}$. Now, let us make the inductive assumption
 \begin{equation}
      \|T(2^{-n}x) - 2^{-n}T(x)\| < 2^{-n/2}k(\|x\|)^{1/2} \left(\sum_{i=0}^{n-1}2^{-i/2}\right) + (1 - 2^{-n})4 \varepsilon 
 \end{equation}
 The inequality (2) is true for $n=1$. Assuming it true for any particular value of $n$, we shall prove it for $n+1$.
 Dividing the inequality (2) by 2, we have 
 \[
      \|T(2^{-n}x)/2 - 2^{-n-1}T(x)\| < 2^{-(n+1)/2}k(\|x\|)^{1/2} \left(\sum_{i=1}^{n}2^{-i/2}\right) + (1/2 - 2^{-n-1})4 \varepsilon 
 \]   
 Replacing $x$ by $2^{-n}x$ in the inequality (1), we get 
 \[
  \|T(2^{-n-1}x) - T(2^{-n}x)/2\| < 2^{-(n+1)/2}k(\|x\|)^{1/2} + 2 \varepsilon
  \]
 Upon adding the last two inequalities, we obtain 
 \[
  \|T(2^{-n-1}x) - 2^{-n-1}T(x)\| < 2^{-(n+1)/2}k(\|x\|)^{1/2}\left( \sum_{i=0}^n 2^{-i/2} \right) + (1 - 2^{-n-1}) 4\varepsilon
 \]
 This proves the induction. Therefore inequality (2) is true for all $x \in E$ and for $n = 1, 2, 3, \ldots$.
 If we put $a = k\sum_{i=0}^{\infty} 2^{-i/2}$, we have
 \[
      \|T(2^{-n}x) - 2^{-n}T(x)\| < 2^{-n/2} a (\|x\|)^{1/2} + 4 \varepsilon
 \]
 Hence, if $m$ and $p$ are any positive integers,
 \[
  \|2^{-m}T(2^mx) - 2^{-m-p}T(2^{m+p}x)\| = 2^{-m}\left\|T\left(2^{m+p}\frac{x}{2p}\right) - 2^{-p}T(2^{m+p}x)\right\| < 2^{-m/2}a(\|x\|)^{1/2} + 2^{2-m}\varepsilon
  \]
  for all $x \in E$. Therefore, since $E$ is a complete space, the limit $U(x) = \lim_{n \to \infty}(T(2^nx)/2^n)$ exists for all $x \in E$.
  \newline
  \newline
  To prove that $U(x)$ is an isometry, let $x$ and $y$ be any two points of $E$. Divide the inequality
  \[
      \left| \|T(2^n x) - T(2^n y) \| - 2^n \|x - y\| \right| < \varepsilon
  \]
  by $2^n$ and take the limit as $n \to \infty$. The result is $\|U(x) - U(y)\| = \|x - y\|$. This completes the proof.
\end{proof}

\subsection{Mathematical Background}
We give brief introductions to the subjects of representation theory, differential topology and geometry, and Lie theory,
stating only those definitions and theorems needed to understand the paper. For more comprehensive background, we encourage
readers to consult any of \citet{Fulton_Harris_2004,etingof2011representation,Hall2015} for representation theory, 
any of \citet{Lee2003,Lee2018} for differential topology and geometry, and \citet{Hall2015} for Lie theory.

\subsubsection{Representation Theory}
Representation theory seeks to extend the theory of linear algebra to groups (and more general objects, such as algebras) 
by associating to each group a \textit{representation}, which is a homomorphism from the group to the space of linear operators 
on that group. As a simplification, we often just think of the representation as an association of an $n \times n$ matrix to 
each group element, in which case we have a \textit{matrix group}.

\begin{definition}[Representation of an associative algebra]
We define a {\it representation} $(\rho, V)$ of an associative algebra $A$ to be a vector space
$V$ with an associated homomorphism $\rho : A \to \text{End}(V)$ where $\text{End}(V)$
denotes the set of endomorphisms of $V$, i.e.\ linear operators from $V$ to itself.
\end{definition}

\begin{definition}[Lie group representation]
  A {\it representation} $(\rho, V)$ of a Lie group $G$ is a homomorphism $\rho : G \to GL(V)$
  where $V$ is a vector space.
\end{definition}

\begin{definition}[Lie algebra representation]
  A {\it representation} $(\rho, V)$ of a Lie algebra $\mathfrak{g}$ is a homomorphism $\rho : \mathfrak{g} \to \mathfrak{gl}(V)$
  where $V$ is a vector space.
\end{definition}

\begin{definition}[Morphism of representations]
A {\it morphism\/} of representations $(\rho_1, V), (\rho_2, W)$ is a map
$\phi : V \to W$ satisfying
\[
  \phi(\rho_1(a)(v)) = \rho_2(a) \phi(v)
\]
for all $a \in A, v \in V$.
\end{definition} 

We can view morphisms as the set of transformations on $V$ 
that preserve {\it equivariance\/} with respect to some pair of representations. 
$\phi$ is also sometimes called an {\it intertwining map}. In other 
words, in equivariant deep learning we seek to learn neural networks $\mathcal{N}$
that are morphisms of representations. In almost equivariant deep learning, we 
seek models $\mathcal{N}$ that are almost morphisms in the sense described in the 
paper intro.

\begin{definition}[Subrepresentation]
  A {\it subrepresentation\/} of $(\rho, V)$ is a subspace $U \subseteq V$ such that 
  $\rho(a)(u) \in U$ for all $a \in A, u \in U$.
\end{definition}

\subsubsection{Differential Topology \& Geometry, Lie Groups, and Lie Algebras}

Smooth manifold theory extends the techniques of calculus to high-dimensional, non-Euclidean 
spaces and those without a preferred coordinate system. In layman's terms, a smooth manifold is 
a mathematical object that is locally homeomorphic to $\R^n$ about every point and which has a 
smooth structure that allows one to perform operations from calculus such as differentiation and integration.
More concretely, we equip a Hausdorff, second countable, and locally Euclidean topologyical space 
with a set of charts, $\{(U_k, \varphi_k)\}_{k=1}^n$ which consist of a neighborhood $U_k$ about each 
point and a homeomorphism $\varphi_k : U_k \to \R^n$. We then define \textit{transition maps},
$\psi \circ \varphi^{-1} : \varphi(U \cap V) \to \psi(U \cap V)$ that allow us to move between charts.
For smooth manifolds, we require that these charts are \textit{smoothly compatible}, i.e. that either $U \cap V = \emptyset$
or $\psi \circ \varphi^{-1}$ is a diffeomorphism.

\begin{definition}[Smooth manifold]
  A {\it smooth manifold} is a Hausdorff, second countable, locally Euclidean topological space, $M$,
  equipped with a smooth structure.
\end{definition}

\begin{definition}[Smooth submersion]
A smooth map of manifolds, $F : M \to N$ is said to be a \textit{smooth submersion} if its differential 
is surjective at each point.
\end{definition}

\begin{definition}[Smooth immersion]
A smooth map of manifolds, $F : M \to N$ is said to be a \textit{smooth submersion} if its differential 
is injective at each point.
\end{definition}

\begin{definition}[Riemannian manifold]
  A {\it Riemannian manifold} is a pair $(M, g)$ where $M$ is a smooth manifold and $g$
  is a choice of Riemannian metric on $M$.
\end{definition}

\begin{definition}[Riemannian metric]
  A {\it Riemannian metric} for a manifold $M$ is a smoothly-varying choice of inner product on the tangent space $T_pM$.
  Equivalently, a {\it Riemannian metric} on $M$ is a smooth covariant 2-tensor field $g \in \mathcal{T}^2(M)$ whose value 
  $g_p$ at each $p \in M$ is an inner product on $T_pM$. 
\end{definition}

\begin{proposition}
  Every smooth manifold admits a Riemannian metric.
\end{proposition}

\begin{definition}[Isometry]
  An {\it isometry} of Riemannian manifolds $(M, g)$ and $(\tilde{M}, \tilde{g})$ is a diffeomorphism $\varphi : M \to \tilde{M}$
  such that $\varphi^* \tilde{g} = g$. Equivalently, $\varphi$ is a metric-preserving diffeomorphism.
\end{definition}

\begin{definition}[Transitive group action]
  A group action on $M$ is said to be \textit{transitive} if for every pair of points $p, q \in M$,
  there exists $g \in G$ such that $g \cdot p = q$ or, equivalently, if the only orbit is all of $M$.
\end{definition}

\begin{theorem}[Global Rank Theorem]
  Let $M$ and $N$ be smooth manifolds, and suppose $F: M \to N$ is a smooth map of constant rank. Then
  \begin{enumerate}
      \item If $F$ is surjective, then it is a smooth submersion.
      \item If $F$ is injective, then it is a smooth immersion.
      \item If $F$ is bijective, then it is a diffeomorphism.
  \end{enumerate}
\end{theorem}

\begin{theorem}[Equivariant Rank Theorem]
  Let $M$ and $N$ be smooth manifolds and let $G$ be a Lie group. Suppose $F: M \to N$ is a smooth 
  map that is equivariant with respect to a transitive smooth $G$-action on $M$ and any smooth $G$-action on $N$.
  Then $F$ has constant rank. Thus, if $F$ is surjective, it is a smooth submersion; if it is injective,
  it is a smooth immersion; and if it is bijective, it is a diffeomorphism.
\end{theorem}

\begin{proposition}
  Suppose $\theta$ is a smooth left action of a Lie group $G$ on a smooth manifold $M$. For each $p \in M$,
  the orbit map $\theta^{(p)} : G \to M$ is smooth and has constant rank, so the isotropy group 
  $G_p = (\theta^{(p)})^{-1}(p)$ is a properly embedded Lie subgroup of $G$. If $G_p = \{e\}$,
  then $\theta^{(p)}$ is an injective smooth immersion, so the orbit $G \cdot p$ is an immersed 
  submanifold of $M$.
\end{proposition}

\begin{definition}[Lie group]
  A {\it Lie group\/} is a smooth manifold with an algebraic group 
  structure such that the multiplication map $m: G \times G \to G$ and 
  the inversion map $i: G \to G$ are both smooth.
\end{definition}

\begin{definition}[Lie algebra]
  A {\it Lie algebra\/} is a vector space $\mathfrak{g}$ over a field $F$, equipped with
  a map $[\cdot, \cdot] : \mathfrak{g} \times \mathfrak{g} \to \mathfrak{g}$, called 
  the {\it bracket}, which satisfies the following three properties:
  \begin{enumerate}
      \item Bilinearity
      \item Antisymmetry $$[X, Y] = -[Y, X]$$
      \item The Jacobi Identity $$[X, [Y, Z]] + [Y, [Z, X]] + [Z, [X, Y]] = 0$$
  \end{enumerate}
\end{definition}

\begin{theorem}[Ado's Theorem]
  Every finite-dimensional real Lie algebra admits a faithful finite-dimensional representation.
\end{theorem}

\begin{definition}[Matrix exponential]
  Given $A \in \mathbb{R}^{n \times n}$, the {\it matrix exponential\/} is the function
  $\exp: \mathbb{R}^{n \times n} \to \mathbb{R}^{n \times n}$ given by 
  \[
      \exp(A) = e^A = \sum_{k=0}^{\infty} \frac{A^k}{k!}
  \]
\end{definition}

\begin{definition}[Haar measure]
  Let $G$ be a locally compact group. Then the (unique up to scalars, nonzero, left-invariant) {\it Haar measure}
  on $G$ is the Borel measure $\mu$ satisfying the following
  \begin{enumerate}
      \item $\mu(x E) = \mu(E)$ for all $x \in G$ and all measurable $E \subseteq G$.
      \item $\mu(U) > 0$ for every non-empty open set $U \subseteq G$.
      \item $\mu(K) < \infty$ for every compact set $K \subseteq G$.
  \end{enumerate}
\end{definition}

\begin{proposition}
  Every Lie group is locally compact and thus comes equipped with a Haar measure.
\end{proposition}

\subsection{Model Training \& Hyperparameter Tuning}
\subsubsection{Pendulum Trajectory Prediction}
For the pendulum trajectory prediction task, we performed a grid search over the following parameters across all
models excluding, to some extent, the standard CNN. For the standard CNN, we used a fixed architecture with three 
convolutional layers having a kernel size of 2 and having 32, 64, and 128 channels, respectively. This was followed by two linear layers
having weight matrices of sizes $128 \times 256$ and $256 \times 2$, respectively.

Each model was given a batch size of 16 and trained for 100 epochs. An 80\%/10\%/10\% train-validation-test split was used, with RMSE calculated on the test set
after the final epoch. The data was not shuffled due to this being a time series prediction task. Four random seeds were used at each step of the grid search, with average test set RMSE and standard deviations 
calculated with respect to the four random seeds. 
\begin{table}[h]
  \centering
  \begin{tabular}{@{}ccccc@{}}
  \toprule
  \textbf{Learning Rate} & \textbf{Optimizer} & \textbf{Kernel Sizes} & \textbf{Hidden Channels} & \textbf{\# Hidden Layers} \\ \midrule
  1e-4, 1e-3, 1e-2, 1e-1 & Adam, SGD          & 2, 3, 4, 5            & 16, 32                   & 1, 2, 3, 4                \\ \bottomrule
  \end{tabular}
  \vspace{0.25cm}
  \caption{Model hyperparameters used in grid search for the pendulum trajectory prediction task.}
  \label{tab:pendulum-grid-search}
  \end{table}

Below, we provide plots of train and validation RMSE as well as training loss for each of the best performing models.

\begin{figure}
\centering
\begin{tabular}{cc}
\includegraphics[width=65mm]{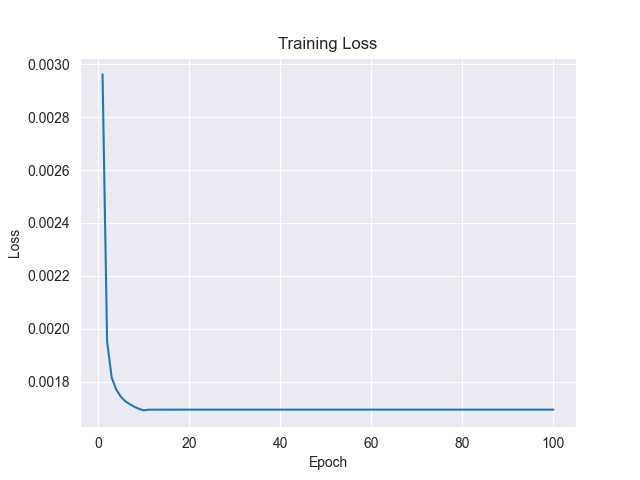} &   \includegraphics[width=65mm]{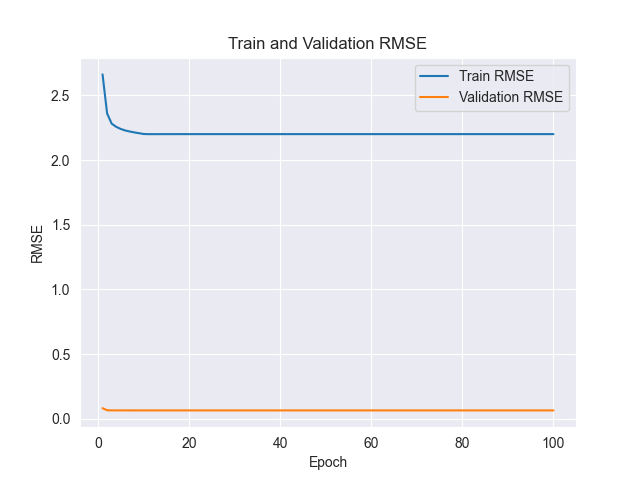} \\
(a) CNN Training Loss & (b) CNN RMSE \\[6pt]
\includegraphics[width=65mm]{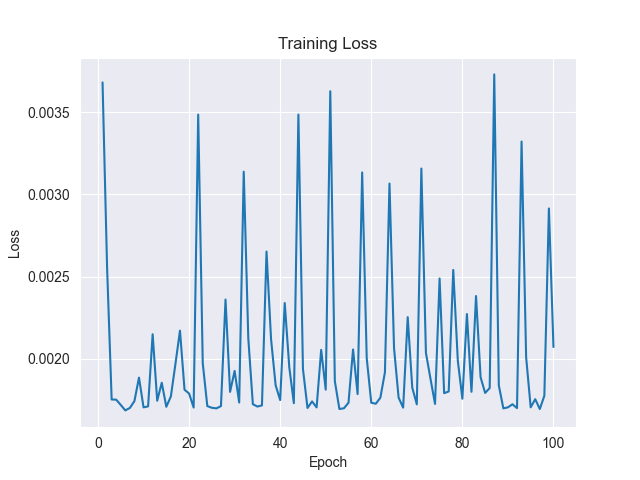} &   \includegraphics[width=65mm]{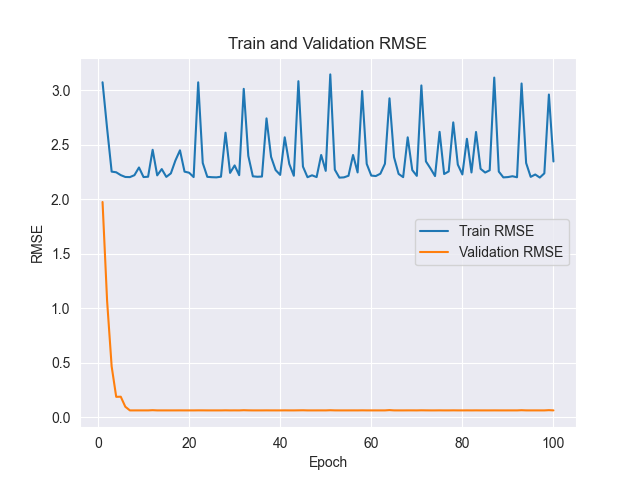} \\
(c) E2CNN Training Loss & (d) E2CNN RMSE \\[6pt]
\includegraphics[width=65mm]{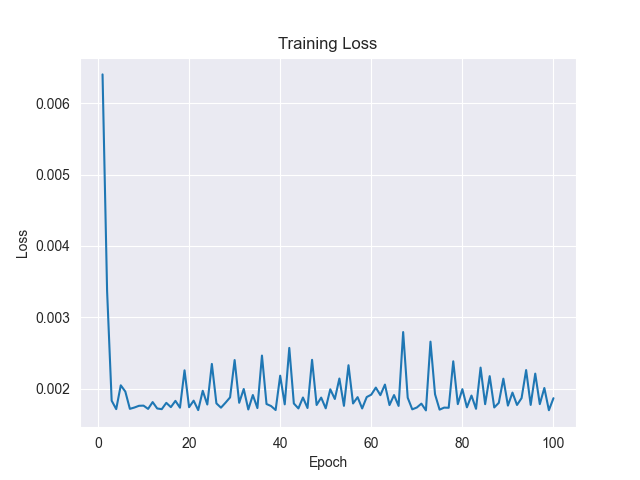} &   \includegraphics[width=65mm]{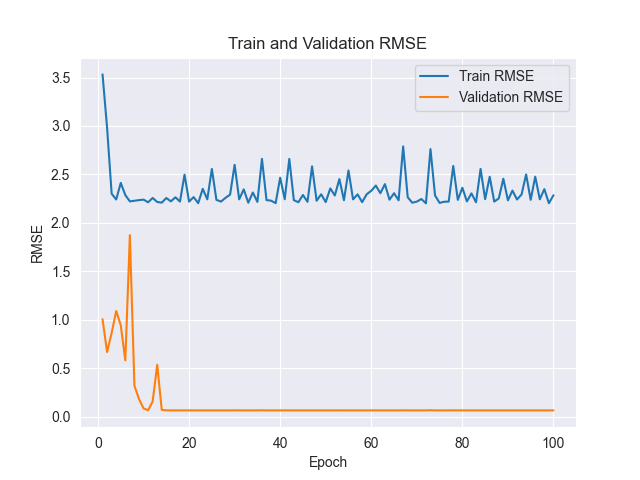} \\
(e) RPP Training Loss & (f) RPP RMSE \\[6pt]
\end{tabular}
\end{figure}

\begin{figure}
\centering
\begin{tabular}{cc}
\includegraphics[width=65mm]{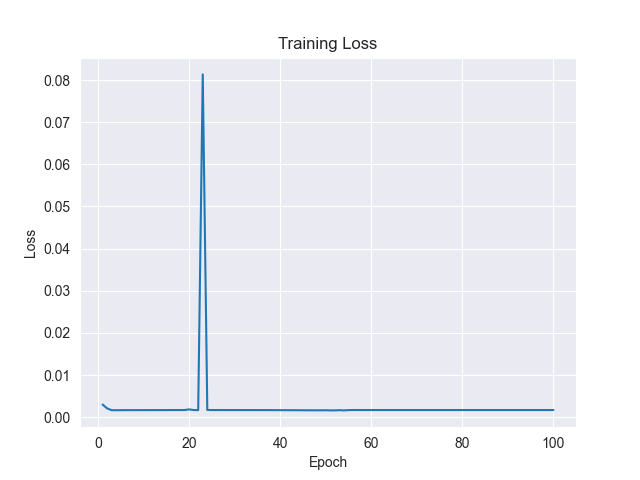} &   \includegraphics[width=65mm]{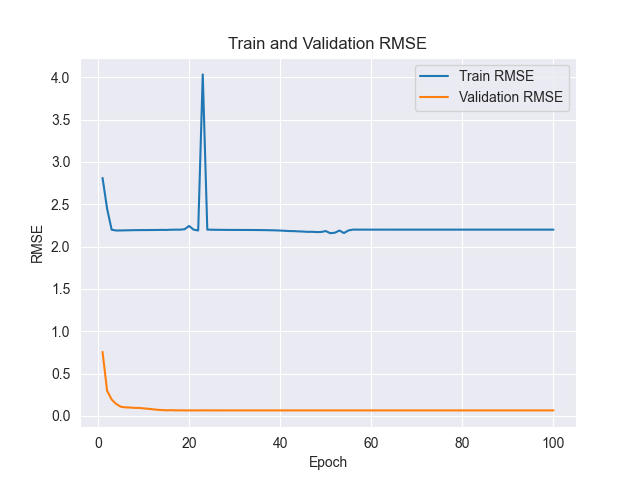} \\
(g) Approximately Equivariant G-CNN Training Loss & (h) Approximately Equivariant G-CNN RMSE \\[6pt]
\includegraphics[width=65mm]{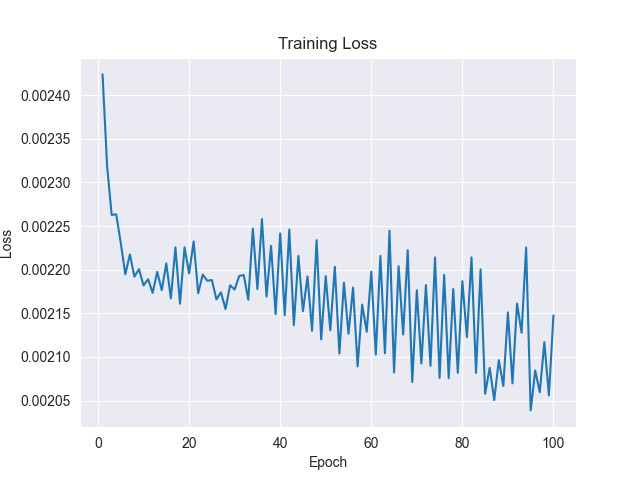} &   \includegraphics[width=65mm]{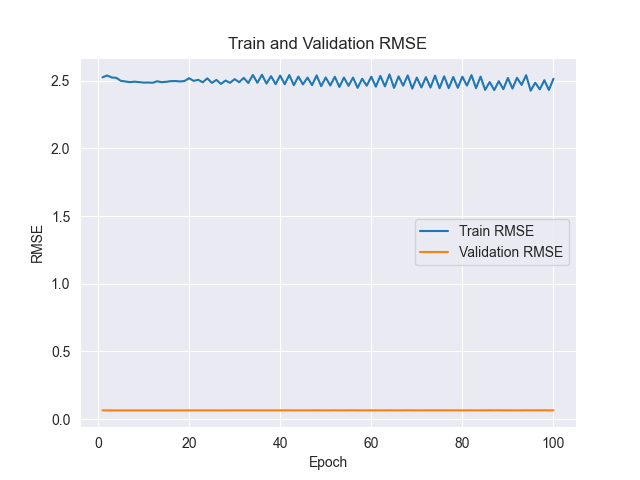} \\
(i) Almost Equivariant G-CNN Training Loss & (j) Almost Equivariant G-CNN RMSE \\[6pt]
\end{tabular}
\caption{Training Losses and Train/Validation RMSE across Epochs for Pendulum Trajectory Prediction}
\end{figure}

\subsubsection{Rotated MNIST Classification}
For the Rotated MNIST classification task, we performed a grid search over the following parameters across all
models excluding the standard CNN. For the standard CNN, we used a fixed architecture with two convolutional layers
having hidden channel counts of 32 and 64, respectively, and a kernel size of 3. The convolutional layers are followed by dropout
and two linear layers having weight matrices of sizes $9126 \times 128$ and $128 \times 10$, respectively.

Each model was trained for 200 epochs with a linear learning rate decay schedule. The standard 10k/2k/50k train-validation-test split was used, with classification accuracy calculated on the test set
after the final epoch.
\begin{table}[h]
  \centering
  \begin{tabular}{@{}cccccc@{}}
  \toprule
  \textbf{Learning Rate} & \textbf{Optimizer} & \textbf{Kernel Sizes} & \textbf{Hidden Channels} & \textbf{\# Hidden Layers} & \textbf{Batch Sizes} \\ \midrule
  1e-4, 1e-3, 1e-2, 1e-1 & Adam               & 3, 4, 5               & 16, 32                   & 1, 2, 3, 4                & 16, 32, 64           \\ \bottomrule
  \end{tabular}
  \vspace{0.25cm}
  \caption{Model hyperparameters used in grid search for the Rot-MNIST classification task.}
  \label{tab:rot-mnist-grid-search}
\end{table}

\subsection{Computational Resources}
All experiments were conducted on a single NVIDIA A100 GPU with 80GB of memory.

\end{document}